\documentclass{article}

\PassOptionsToPackage{numbers, compress}{natbib}

\usepackage[acronym,smallcaps,nowarn]{glossaries}
\glsdisablehyper{}

\usepackage[preprint]{neurips_2021}
\usepackage[utf8]{inputenc} 
\usepackage[T1]{fontenc}    
\usepackage{hyperref}       
\usepackage{url}            
\usepackage{booktabs}       
\usepackage{amsfonts}       
\usepackage{nicefrac}       
\usepackage{microtype}      
\usepackage{amsmath}
\usepackage{algorithm,algorithmic}
\usepackage{graphicx}
\usepackage{amsmath,amssymb,amsthm}
\usepackage{algorithm,algorithmic}
\usepackage{subfigure}
\usepackage[outercaption]{sidecap}
\makeatletter
\def\SC@figure@vpos{m}
\makeatother

\usepackage{xcolor}
\usepackage{framed}
\usepackage{graphicx}
\usepackage{mathtools}
\graphicspath{ {images/} }
\usepackage{tikz}
\usetikzlibrary{arrows,shapes,snakes,automata,backgrounds,petri}
\usepackage{pdfpages}
\usepackage{hyperref}
\hypersetup{
	colorlinks=true,
	linkcolor=blue,
	filecolor=magenta,      
	urlcolor=cyan,
}

\usepackage[T1]{fontenc}
\usepackage[utf8]{inputenc}
\usepackage{tabularx,ragged2e,booktabs,caption}
\newcolumntype{C}[1]{>{\Centering}m{#1}}

\usepackage{caption}

\usepackage{enumitem}

\usepackage{thmtools}

\usepackage{hyperref}
\usepackage{cleveref}

\usepackage{mathrsfs}
\usepackage{mathtools}

\usepackage{enumitem}
\usepackage{natbib}
\usepackage{verbatim}
\usepackage{url}
\usepackage{thm-restate}

\usepackage{wrapfig}

\usepackage{bbm}

\usepackage{parskip}

\usepackage{graphbox}

\usepackage{mathrsfs}
\newcommand{\probspace}{\mathscr{P}}

\usepackage{subfigure}

\newtheoremstyle{definition}
{3pt} 
{3pt} 
{} 
{} 
{\bfseries} 
{.} 
{.5em} 
{} 

\theoremstyle{definition}

\renewcommand{\epsilon}{\varepsilon}
\renewcommand{\hat}{\widehat}
\renewcommand{\tilde}{\widetilde}
\renewcommand{\bar}{\overline}

\title{The Nature of Temporal Difference Errors in\\ Multi-step Distributional Reinforcement Learning}

\author{%
  Yunhao Tang \\
  DeepMind\\
  \texttt{robintyh@deepmind.com}
    \And
    Mark Rowland \\
    DeepMind \\
    \texttt{markrowland@deepmind.com}
    \And
    R\'emi Munos \\
    DeepMind \\
    \texttt{munos@deepmind.com}
    \And
    Bernardo \'Avila Pires \\
    DeepMind \\
    \texttt{bavilapires@deepmind.com}
    \And
    Will Dabney \\
    DeepMind \\
    \texttt{wdabney@deepmind.com}
    \And
    Marc G. Bellemare \\
    Google Brain \\
    \texttt{bellemare@google.com}
}

\begin{document}

\maketitle

\begin{abstract}

We study the multi-step off-policy learning approach to distributional RL. Despite the apparent similarity between value-based RL and distributional RL, our study reveals intriguing and fundamental differences between the two cases in the multi-step setting. We identify a novel notion of path-dependent distributional TD error, which is indispensable for principled multi-step distributional RL. 
The distinction from the value-based case bears important implications on concepts such as backward-view algorithms. Our work provides the first theoretical guarantees on multi-step off-policy distributional RL algorithms, including results that apply to the small number of existing approaches to multi-step distributional RL. In addition, we derive a novel algorithm, Quantile Regression-Retrace, which leads to a deep RL agent QR-DQN-Retrace that shows empirical improvements over QR-DQN on the Atari-57 benchmark. Collectively, we shed light on how unique challenges in multi-step distributional RL can be addressed both in theory and practice.

\end{abstract}

\section{Introduction}

The return $\sum_{t=0}^\infty\gamma^t R_t$ is a fundamental concept in reinforcement learning (RL). In general, the return is a random variable, whose distribution captures important information such as the stochasticity in future events. While the classic view of value-based RL typically focuses on the expected return \citep{bertsekas1996neuro,sutton1998,szepesvari2010algorithms}, learning the full return distribution is of both theoretical and practical importance \citep{morimura2010nonparametric,morimura2012parametric,bellemare2017distributional,yang2019fully,bodnar2019quantile,wurman2022outracing,bdr2022}. 

To design efficient algorithms for learning return distributions, a natural idea is to construct distributional equivalents of existing multi-step off-policy value-based algorithms.
In value-based RL, multi-step learning tends to propagate useful information more efficiently and off-policy learning is ubiquitous in modern RL systems.
Meanwhile, the return distribution shares inherent commonalities with the expected return, thanks to the close connection between the distributional Bellman equation \citep{morimura2010nonparametric,morimura2012parametric,bellemare2017distributional,bdr2022} and the celebrated value-based Bellman equation \citep{sutton1998}.
The Bellman equation is foundational to value-based RL algorithms, including many multi-step off-policy methods \citep{precup2001off,harutyunyan2016q,munos2016safe,mahmood2017multi}. Due to the apparent similarity between distributional and value-based Bellman equations, should we expect key value-based concepts and algorithms to seamlessly transfer to distributional learning?

Our study indicates that the answer is no. There are critical differences between distributional and value-based RL, which requires a distinct treatment of multi-step learning. Indeed, thanks to the focus on expected returns, the value-based setup offers many unique conceptual and computational simplifications in algorithmic design. 
However, we find that such simplifications do not hold for distributional learning. Multi-step distributional RL requires a deeper look at the connections between fundamental concepts such as $n$-step returns, TD errors and importance weights for off-policy learning. To this end, we make the following conceptual, theoretical and algorithmic contributions: 

\paragraph{Distributional TD error.} We demonstrate the emergence of a novel notion of path-dependent distributional TD error (Section~\ref{sec:understanding}). Intriguingly, as the name suggests,  path-dependent distributional TD errors are \emph{path-dependent}, i.e., distributional TD errors at time $t$ depend on the sequence of immediate rewards $(R_s)_{s=0}^{t-1}$. This differs from value-based TD errors, which are path-independent. We will show that the path-dependency property is not an artifact, but rather a fundamental property of distributional learning. We show numerically that naively constructing certain path-independent distributional TD errors does not produce convergent algorithms. The path-dependency property also has conceptual and computational impacts on forward-view estimates and backward-view algorithms.

\paragraph{Theory of multi-step distributional RL.} We derive distributional Retrace, a novel and generic multi-step off-policy operator for distributional learning.
We prove that distributional Retrace is contractive and has the target return distribution as its fixed point. Distributional Retrace interpolates between the one-step distributional Bellman operator \citep{bellemare2017distributional} and Monte-Carlo (MC) estimation with importance weighting \citep{chandak2021universal}, trading-off the strengths from the two extremes.

\paragraph{Approximate multi-step distributional RL.} Finally, we derive Quantile Regression-Retrace, a novel algorithm combining distributional Retrace with quantile representations of distributions \citep{dabney2018distributional} (Section~\ref{sec:parametric}). One major technical challenge is to define the quantile regression (QR) loss against signed measures, which are unavoidable in sample-based settings. We bypass the issue of ill-defined QR loss and derive unbiased stochastic estimates to the QR loss gradient. This leads up to QR-DQN-Retrace, a deep RL agent with performance improvements over QR-DQN on  Atari-57 games.

In Figure~\ref{fig:illustration}, we illustrate how the back-up target is computed for multi-step distributional RL. In summary, we take our findings to demonstrate how the set of unique challenges presented by multi-step distributional RL can be addressed both theoretically and empirically. Our study also opens up many exciting research pathways in this domain, paving
the way for future investigations.

\section{Background}\label{sec:background}

Consider a Markov decision process (MDP) represented as the tuple $\left(\mathcal{X},\mathcal{A},P_R,P,\gamma\right)$ where $\mathcal{X}$ is the state space, $\mathcal{A}$ the action space, $P_R:\mathcal{X}\times\mathcal{A}\rightarrow \probspace(\mathscr{R})$ the reward kernel (with $\mathscr{R}$ a finite set of possible rewards), $P:\mathcal{X}\times\mathcal{A}\rightarrow\probspace(\mathcal{X})$ the transition kernel and $\gamma\in [0,1)$ the discount factor. In general, we use $\probspace(A)$ denote a distribution over set $A$.
We assume the reward to take a finite set of values mainly because it is notationally simpler to present results; it is straightforward to extend our results to the general case.
Let $\pi:\mathcal{X}\rightarrow\probspace(\mathcal{A})$ be a fixed policy.
We use $(X_t,A_t,R_t)_{t=0}^\infty\sim\pi$ to denote a random trajectory sampled from $\pi$, such that $A_t\sim \pi(\cdot|X_t),R_t\sim P_R(\cdot|X_t,A_t),X_{t+1}\sim P(\cdot|X_t,A_t)$. Define $G^\pi(x,a)\coloneqq\sum_{t=0}^\infty \gamma^t R_t$ as the random return, obtained by following $\pi$ starting from $(x,a)$. The Q-function $Q^\pi(x,a)\coloneqq\mathbb{E}[G^\pi(x,a)]$ is defined as the expected return under policy $\pi$. For convenience, we also adopt the vector notation $Q\in\mathbb{R}^{\mathcal{X}\times\mathcal{A}}$.
Define the one-step value-based Bellman operator $T^\pi:\mathbb{R}^{\mathcal{X\times\mathcal{A}}}\rightarrow\mathbb{R}^{\mathcal{X\times\mathcal{A}}}$ such that $T^\pi Q(x,a) \coloneqq \mathbb{E}[R_0+\gamma Q\left(X_1,A_1^\pi\right)|X_0=x,A_0=a]$ where $Q(X_t,A_t^\pi)\coloneqq\sum_a \pi(a|X_t)Q(X_t,a)$.  The Q-function $Q^\pi$ satisfies  $Q^\pi=T^\pi Q^\pi$ and is also the unique fixed point of $T^\pi$.

\begin{figure}[t]
    \centering
    \includegraphics[keepaspectratio,width=.8\textwidth]{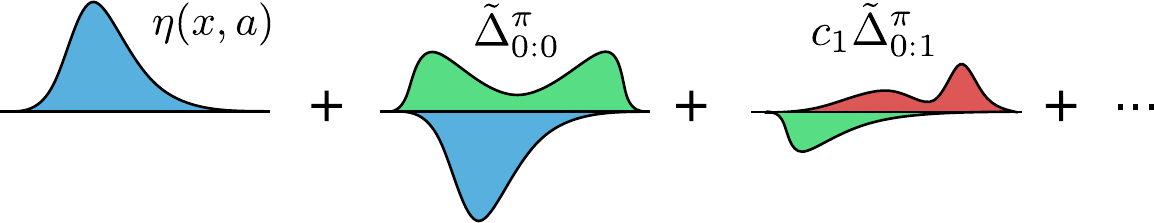}
    \caption{Illustration of a multi-step distributional RL target, constructed as a sum of the initial distribution (left) and weighted distributional TD errors $\tilde{\Delta}^\pi_{0:0}, c_1 \tilde{\Delta}^\pi_{0:1}, \ldots$ across multiple time steps (middle and right); see Section~\ref{sec:dist} for further details and notation. In general, distributional TD errors are signed measures, as reflected by the downwards probability mass; they are also scaled by trace coefficients $c_1$ to correct for off-policy discrepancies between target and behavior policy.}
    \label{fig:illustration}
\end{figure}

\subsection{Distributional reinforcement learning}

In general, the return $G^\pi(x,a)$ is a random variable and we define its distribution as $\eta^\pi(x,a)\coloneqq \text{Law}_\pi\left(G^\pi(x,a)\right)$. 
The return distribution satisfies the distributional Bellman equation \citep{morimura2010nonparametric,morimura2012parametric,bellemare2017distributional,rowland2018analysis,bdr2022},
\begin{align}
     \eta^\pi(x,a) = \mathbb{E}_\pi\left[ \left(\textrm{b}_{R_0,\gamma}\right)_{\#} \eta^\pi\left(X_1,A_1^\pi\right)\; \middle| \; X_0=x,A_0=a\right] \, , \label{eq:dist-bellman}
\end{align}
where $(\textrm{b}_{r,\gamma})_\#:\probspace(\mathbb{R})\rightarrow \probspace(\mathbb{R})$ is the pushforward operation defined through the function $\textrm{b}_{r,\gamma}(z)=r+\gamma z$ \citep{rowland2018analysis}. For convenience, we adopt the notation $\eta^\pi(X_t,A_t^\pi)\coloneqq\sum_a \pi(a|X_t)\eta^\pi(X_t,a)$.
Throughout the paper, we focus on the space of distributions with bounded support $\probspace_\infty(\mathbb{R})$. 
Let $\eta\in\probspace_\infty(\mathbb{R})^{\mathcal{X}\times\mathcal{A}}$ be any distribution vector, we define the \emph{distributional Bellman operator} $\mathcal{T}^\pi: \probspace_\infty(\mathbb{R})^{\mathcal{X}\times\mathcal{A}}\rightarrow\probspace_\infty(\mathbb{R})^{\mathcal{X}\times\mathcal{A}}$ as follows \citep{rowland2018analysis,bdr2022},
\begin{align}
    \mathcal{T}^\pi\eta(x,a)\coloneqq \mathbb{E}\left[(\textrm{b}_{R_0,\gamma})_\# \eta(X_1,A_1^\pi)\; \middle| \; X_0=x,A_0=a\right]\, .\label{eq:dist-bellman-op}
\end{align}
Let $\eta^\pi$ be the collection of return distributions under $\pi$; 
the distributional Bellman equation can then be rewritten as
$
    \eta^\pi = \mathcal{T}^\pi \eta^\pi
$. The distributional Bellman operator $\mathcal{T}^\pi$ is $\gamma$-contractive under the supremum $p$-Wasserstein distance \citep{dabney2018distributional,bdr2022}, so that $\eta^\pi$ is the unique fixed point of $\mathcal{T}^\pi$. See Appendix~\ref{appendix:metric} for details of the distance metrics.

\subsection{Multi-step off-policy value-based learning}

We provide a brief background on the value-based multi-step off-policy setting as a reference for the distributional case discussed below. In off-policy learning,  the data is generated under a behavior policy $\mu$, which potentially differs from target policy $\pi$. The aim is to evaluate the target Q-function $Q^\pi$.
As a standard assumption, we require $\text{supp}(\pi(\cdot|x))\subseteq\text{supp}(\mu(\cdot|x)),\forall x\in\mathcal{X}$. Let $\rho_t\coloneqq \pi(A_t|X_t)/\mu(A_t|X_t)$ be the step-wise importance sampling (IS) ratio at time step $t$. Step-wise IS ratios are critical in correcting for the off-policy discrepancy between $\pi$ and $\mu$.

Let $c_t\in [0,\rho_t]$ be a time-dependent trace coefficient. We denote $c_{1:t}=c_1\cdots c_t$ and define $c_{1:0}=1$ by convention. Consider a generic form of the return-based off-policy operator $R^{\pi,\mu}$ as in \citep{munos2016safe},
\begin{align}
    R^{\pi,\mu}Q(x,a)\coloneqq Q(x,a) + \mathbb{E}_{\mu}\left[ \sum_{t=0}^\infty c_{1:t}\gamma^t\underbrace{\left(R_t+\gamma Q\left(X_{t+1},A_{t+1}^\pi\right) - Q(X_t,A_t)\right)}_{\delta_t^\pi=\text{value-based\ TD\ error}}\right],\label{eq:retrace}
\end{align}
In the above and below, we omit the notation conditioning on $X_0=x,A_0=a$ for conciseness. The general form of $R^{\pi,\mu}$ encompasses many important special cases: when on-policy and $c_t=\lambda$, it recovers the Q-function variant of TD($\lambda$) \citep{sutton1998,harutyunyan2016q}; when $c_t=\lambda \min(\bar{c},\rho_t)$, it recovers a specific form of Retrace \citep{munos2016safe}; when $c_t=\rho_t$, it recovers the importance sampling (IS) operator. The back-up target is computed as a mixture over TD errors $\delta_t^\pi$, each calculated from the one-step transition data. We also define the \emph{discounted TD error} $\tilde{\delta}_{t}^\pi=\gamma^t \delta_t^\pi$, which can be interpreted as the difference between $n$-step returns from two time steps $t$ and $t+1$, as we discuss in Section~\ref{sec:understanding}. As we will detail, the property of $\tilde{\delta}_t^\pi$ marks a significant difference from the distributional RL setting.

By design,
$R^{\pi,\mu}$ has $Q^\pi$ as the unique fixed point. Multi-step updates make use of rewards from multiple time steps, propagating learning signals more efficiently. This is reflected by the fact that $R^{\pi,\mu}$ is $\beta$-contractive with $\beta\in[0,\gamma]$ \citep{munos2016safe} and often contracts to $Q^\pi$ faster than the one-step Bellman operator $T^\pi$. Our goal is to design distributional equivalents of multi-step off-policy operators, which can lead to concrete algorithms with sample-based learning. 

\section{Multi-step off-policy distributional reinforcement learning}
\label{sec:dist}

We now present the core theoretical results relating to multi-step distributional operators. In general, the aim is to evaluate the target distribution $\eta^\pi$ with access to off-policy data generated under $\mu$.

Below, we use $G_{t':t}=\sum_{s=t'}^t \gamma^{s-t'} R_s$ to denote the partial sum of discounted rewards between two time steps $t'\leq t$. We define the generic form of multi-step off-policy distributional operator $\mathcal{R}^{\pi,\mu}$ such that for any $\eta\in\probspace_\infty(\mathbb{R})^{\mathcal{X}\times\mathcal{A}}$, its back-up target $\mathcal{R}^{\pi,\mu}\eta(x,a)$ is computed as 
\begin{align}
    \eta(x,a) + \mathbb{E}_{\mu}\left[ \sum_{t=0}^\infty c_{1:t} \cdot \left(\underbrace{\left(\textrm{b}_{G_{0:t},\gamma^{t+1}}\right)_{\#}\eta\left(X_{t+1},A_{t+1}^\pi\right) - \left(\textrm{b}_{G_{0:t-1},\gamma^{t}}\right)_{\#}\eta(X_t,A_t)}_{\tilde{\Delta}_{0:t}^\pi=\text{Multi-step\ Distributional\ TD\ error}}\right)\right].\label{eq:dist-retrace-operator}
\end{align}

As an effort to simplify the naming, we  call $\mathcal{R}^{\pi,\mu}$ the \emph{distributional Retrace} operator. Distributional Retrace only requires $c_t\in[0,\rho_t]$ and represents a large family of distributional operators. Throughout, we will heavily adopt the pushforward notations. This is mainly because instead of directly working with the random variable $G^\pi$, we find it much more convenient to express various important multi-step operations with pushfoward notations.

The back-up target $\mathcal{R}^{\pi,\mu}\eta(x,a)$ is written as a weighted sum of the path-dependent distributional TD errors $\tilde{\Delta}_{0:t}^\pi$, which we extensively discuss in Section~\ref{sec:understanding}. 
Though the form of $\mathcal{R}^{\pi,\mu}$ seems to bear certain similarities to the value-based operator in Equation~\eqref{eq:retrace}, the critical differences lie in subtle definitions of the distributional TD errors $\tilde{\Delta}_{0:t}^\pi$ and where to place the traces $c_{1:t}$ for off-policy corrections. We resume to unpack the insights entailed by the design of the operator in Section~\ref{sec:understanding}.

Below, we first present theoretical properties of the distributional Retrace operator. We start with a key property which underlies many ensuing theoretical results. Given a fixed $n$-step reward sequence $r_{0:n-1}$ and a fixed state-action pair $(x,a)\in\mathcal{X}\times\mathcal{A}$, we call pushfoward distributions of the form $\left(\textrm{b}_{\sum_{s=0}^{n-1}\gamma^s r_s,\gamma^{n}}\right)_\#\eta(x,a)$ the \emph{$n$-step target distributions}. Our result shows that the back-up target of Retrace is a convex combination of $n$-step target distributions with varying values of $n$.

\begin{restatable}{lemma}{lemmaconvexcombination}\label{lemma:convexcombination} (\textbf{Convex combination}) The Retrace back-up target  is a convex combination of $n$-step target distributions. Formally, there exists an index set $I(x,a)$ such that $\mathcal{R}^{\pi,\mu}\eta(x,a)=\sum_{i\in I(x,a)} w_i \eta_i$ where $w_i\geq 0$,  $\sum_{i\in I(x,a)}w_i=1$ and $(\eta_i)_{i\in I(x,a)}$ are $n_i$-return target distributions.
\end{restatable}
 Since $\mathcal{R}^{\pi,\mu} \eta\in\probspace_\infty(\mathbb{R})^{\mathcal{X}\times\mathcal{A}}$, we can measure the contraction of $\mathcal{R}^{\pi,\mu}$ under probability metrics.

\begin{restatable}{proposition}{propretracecontractive}\label{prop:retracecontractive} (\textbf{Contraction}) $\mathcal{R}^{\pi,\mu} $ is $\beta$-contractive under supremum $p$-Wasserstein distance, where
$
    \beta=\max_{x\in\mathcal{X},a\in\mathcal{A}}\sum_{t=1}^\infty \mathbb{E}_\mu\left[ c_1...c_{t-1}(1-c_t) \right] \gamma^{t} \leq \gamma$.
\end{restatable} 

The contraction rate of the distributional Retrace operator is determined by its effective horizon. At one extreme, when $c_t=0$, the effective horizon is $1$ and $\beta=\gamma$, in which case Retrace recovers the one-step operator. At the other extreme, when $c_t=\rho_t$, the effective horizon is infinite which gives $\beta=0$. This latter case can be understood as correcting for all the off-policy discrepancies with IS, which is very efficient \emph{in expectation} but incurs high variance under sample-based approximations. Proposition~\ref{prop:retracecontractive} also
implies that the distributional Retrace operator has a unique fixed point.

\begin{restatable}{proposition}{propretracefixedpoint}\label{prop:retracefixedpoint} (\textbf{Unique fixed point}) $\mathcal{R}^{\pi,\mu}$ has $\eta^\pi$ as the unique fixed point in $\probspace_\infty(\mathbb{R})^{\mathcal{X}\times\mathcal{A}}$.
\end{restatable}

The above result suggests that starting with  $\eta_0\in\probspace_\infty(\mathbb{R})^{\mathcal{X}\times\mathcal{A}}$, the recursion $\eta_{k+1}=\mathcal{R}^{\pi,\mu}\eta_k$ produces  iterates $(\eta_k)_{k=0}^\infty\in \probspace_\infty(\mathbb{R})^{\mathcal{X}\times\mathcal{A}}$ which converge to $\eta^\pi$ in $\bar{W}_p$ at a rate of $\mathcal{O}(\beta^k)$.

\section{Understanding multi-step distributional reinforcement learning}
\label{sec:understanding}

Now, we pause and take a closer look at the construction of the distributional Retrace operator. We present a number of insights that distinguish distributional learning from value-based learning. 

\subsection{Path-dependent TD error}

The value-based Retrace back-up target can be written as a mixture of value-based TD errors. To better parse the distributional Retrace operator and draw comparison to the value-based setting, we seek to rewrite the distributional back-up target $\mathcal{R}^{\pi,\mu}\eta(x,a)$ into a weighted sum of some notion of distributional TD errors.
To this end, we start with a natural analogy to the value-based TD error.
\begin{restatable}{definition}{definedisttderror}\label{define:disttderror} (\textbf{Distributional TD error}) Given a transition $(X_t,A_t,R_t,X_{t+1})$, define the associated distributional TD error as $\Delta^\pi(X_t,A_t,R_t,X_{t+1})\coloneqq (\textrm{b}_{R_t,\gamma})_\# \eta\left(X_{t+1},A_{t+1}^\pi\right) - \eta(X_t,A_t)$.
\end{restatable}

When the context is clear, we also adopt the concise notation $\Delta_t^\pi=\Delta^\pi(X_t,A_t,R_t,X_{t+1})$. By construction, distributional TD errors are signed measures with zero total mass
\citep{bdr2022}. The distributional TD error is a natural counterpart to the value-based TD error, because they both stem directly from the corresponding one-step Bellman operators. However, unlike in value-based RL, where TD errors alone suffice to specify the multi-step learning operator (Equation~\eqref{eq:retrace}), in distributional RL this is not enough. We introduce the path-dependent distributional TD error, which serves as the building block to distributional Retrace.

\begin{restatable}{definition}{definedisttderror}\label{define:disttderror} (\textbf{Path-dependent distributional TD error}) Given a trajectory $(X_s,A_s,R_s)_{s=0}^\infty$, define the path-dependent distributional TD error at time $t\geq 0$ as follows,
\begin{align}
    \tilde{\Delta}_{0:t}^\pi\coloneqq \left(\textrm{b}_{G_{0:t-1},\gamma^{t}}\right)_{\#}\Delta_t^\pi\label{eq:dist-td-error}.
\end{align}
\end{restatable}
Path-dependent distributional TD errors are defined as a pushforward measures from $\Delta_t^\pi$, where the pushforward operations depend on $G_{0:t-1}$. This equips $\tilde{\Delta}_{0:t}^\pi$ with an intriguing property, \emph{path-dependency}. Concretely, this means that the path-dependent distributional TD error depends on the sequence of rewards $(R_s)_{s=0}^{t-1}$ leading up to step $t$. With the above definitions, we can finally rewrite the back-up target of distributional Retrace as a weighted sum of path-dependent distributional TD errors $\mathcal{R}^{\pi,\mu}\eta(x,a)=\eta(x,a)+\mathbb{E}_\mu[\sum_{t=0}^\infty c_{1:t} \tilde{\Delta}_{0:t}^\pi]$. We now illustrate the difference between value-based and distributional TD errors.

\paragraph{Comparison with value-based TD equivalents.} The value-based equivalent to the path-dependent distributional TD error is the  discounted value-based TD error $\tilde{\delta}_t^\pi = \gamma^t \delta_t^\pi$ which we briefly mentioned in Section~\ref{sec:background}. To see why, note that discounted value-based TD errors allow us to rewrite the value-based Retrace back-up target as $R^{\pi,\mu}Q(x,a)=Q(x,a)+\mathbb{E}_\mu[\sum_{t=0}^\infty c_{1:t} \tilde{\delta}_t^\pi]$. For direct comparison between the two settings, we rewrite both $\tilde{\Delta}_{0:t}^\pi$ and $\tilde{\delta}_t^\pi$ as the difference between two $n$-step predictions evaluated at two time steps $t$ and $t+1$, 
\begin{align}
& \tilde{\Delta}_{0:t}^\pi = \left(\textrm{b}_{G_{0:t},\gamma^{t+1}}\right)_\# \eta\left(X_{t+1},A_{t+1}^\pi\right) - \left(\textrm{b}_{G_{0:t-1},\gamma^{t}}\right)_\# \eta(X_t,A_t), \tag{\text{Distributional}}\label{eq:multistep-dist-td-error} \\
 & \tilde{\delta}_t^\pi =  \left(G_{0:t} + \gamma^{t+1} Q\left(X_{t+1},A_{t+1}^\pi\right) \right)-\left( G_{0:t-1} + \gamma^t Q(X_t,A_t)\right). \tag{\text{Value-based}}\label{eq:multistep-value-td-error}
\end{align}
The above rewriting attributes the path-dependency to the fact that the $n$-step distributional prediction $\left(\textrm{b}_{G_{0:t},\gamma^{t+1}}\right)_\# \eta\left(X_{t+1},A_{t+1}^\pi\right)$ is non-linear in  $G_{0:n-1}$. Indeed, in the value-based setting, because $G_{0:t}=G_{0:t-1}+\gamma^t R_t$ the partial sum of rewards $G_{0:t-1}$ cancels out as a common term. This leaves the discounted TD error $\tilde{\delta}_t^\pi$ \emph{path-independent}. In other words, the computation of $\tilde{\delta}_t^\pi$ does not depend on past rewards $(R_s)_{s=0}^{t-1}$. In contrast, in the distributional setting, the pushforward operations are non-linear in the partial sum of rewards $G_{0:t-1}$. As a result, $G_{0:t-1}$ does not cancel out in the definition of $\tilde{\Delta}_{0:t}^\pi$, making the path-dependent TD error $\tilde{\Delta}_{0:t}^\pi$ depend on the past rewards $(R_s)_{s=0}^{t-1}$. 

The path-dependent property is not an artifact of the distributional Retrace operator $\mathcal{R}^{\pi,\mu}$; instead, it is an indispensable element for convergent multi-step distributional learning in general. We show this by empirically verifying that multi-step learning operators based on alternative definitions of \emph{path-independent} distributional TD errors are non-convergent even for simple problems.

\begin{wrapfigure}{r}{0.35\textwidth}
  \centering
  \includegraphics[width=0.31\textwidth]{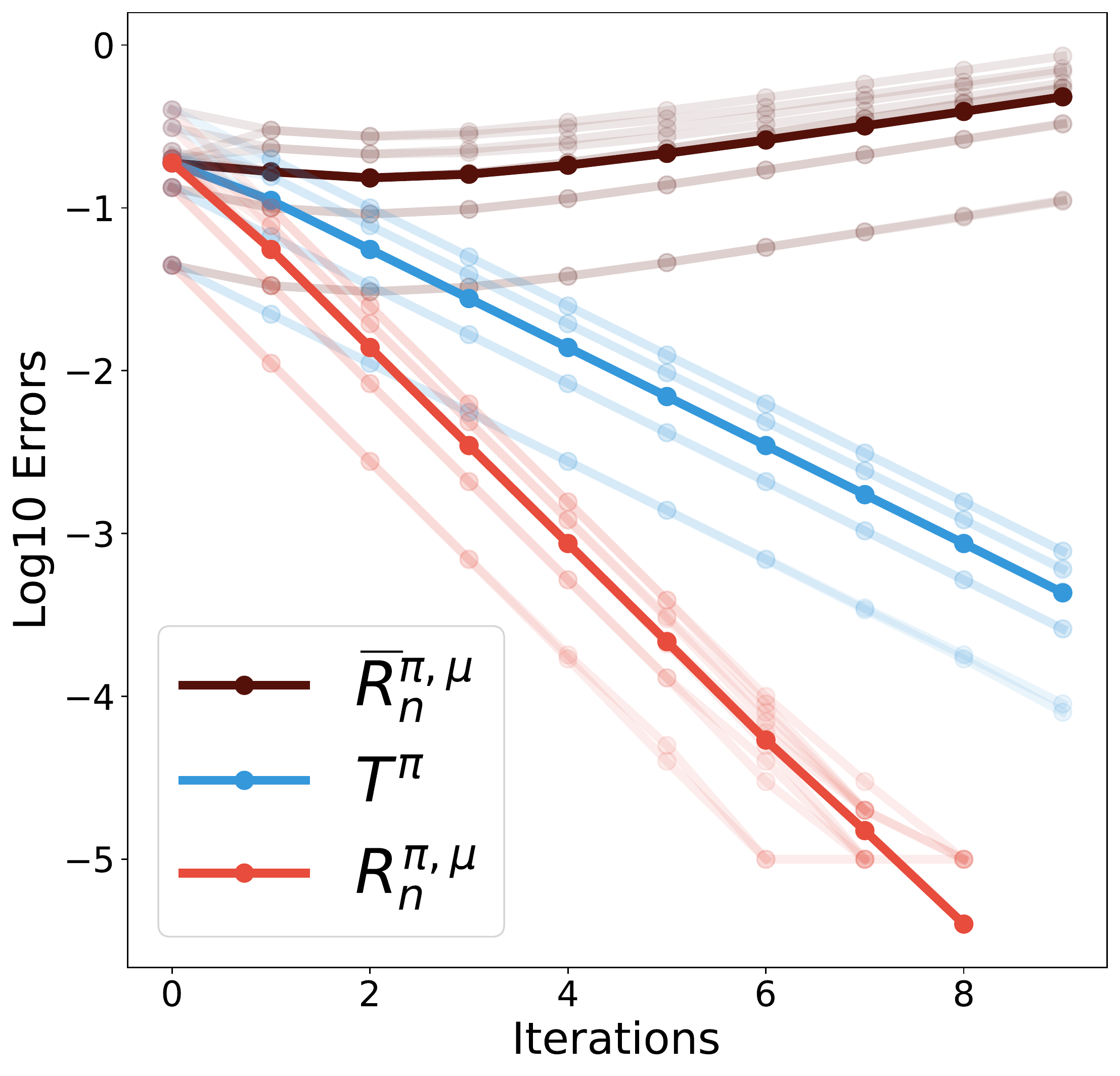}
  \caption{Non-convergent example: comparing $L_p(\mathcal{R}^k \eta_0,\eta^\pi)$ across iterations. We plot $10$ randomly initialized runs. Note $(\bar{\mathcal{R}}^{\pi,\mu})^k\eta_0$ does not converge to $\eta^\pi$ while others do.}
  \label{fig:counterexample}
  \vspace{-1.2cm}
\end{wrapfigure}

\paragraph{Numerically non-convergent path-independent operators.}   Consider the \emph{path-independent} distributional TD error $
    \bar{\Delta}_t^\pi \coloneqq \left(\textrm{b}_{0,\gamma^t}\right)_\# \Delta_t^\pi$.
    We arrived at this definition by dropping the path-dependent term $G_{0:t-1}$ in the pushforward of $\tilde{\Delta}_{0:t}^\pi$. Such a definition seems appealing because when $\eta=\eta^\pi$, the error is zero in expectation $\mathbb{E}_\mu\left[\bar{\Delta}_t^\pi|X_t,A_t\right]=0$. This implies that we can construct a multi-step operator by a weighted sum of the alternative path-independent TD error $
   \bar{\mathcal{R}}_n^{\pi,\mu}\eta(x,a)\coloneqq \eta(x,a) + \mathbb{E}_\mu\left[\sum_{t=0}^\infty c_{1:t} \bar{\Delta}_t^\pi\right]$.
   By construction, $ \bar{\mathcal{R}}_n^{\pi,\mu}$ has $\eta^\pi$ as one fixed point.
 
We provide a very simple counterexample on which $\bar{\mathcal{R}}^{\pi,\mu}$ is not contractive: consider an MDP with one state and one action. The state transitions back to itself with a deterministic reward $R_t=1$. When the discount factor is $\gamma=0.5$, $\eta^\pi$ is a Dirac distribution centered at $2$. We consider the simple case $c_1=\rho_1$ and $c_t=0,\forall t\geq 2$. We use the $L_p$ distance to measure the convergence of the distribution iterates \citep{bdr2022}. Figure~\ref{fig:counterexample} shows that $(\bar{\mathcal{R}}_n^{\pi,\mu})^k\eta_0$ does not converge to $\eta^\pi$, while the one-step Bellman operator $\mathcal{T}^\pi$ and distributional Retrace  $\mathcal{R}^{\pi,\mu}$ are convergent. 

In Appendix~\ref{appendix:failure}, we discuss yet another alternative to $\tilde{\Delta}_{0:t}^\pi$ designed to be path-independent $\gamma^t \Delta_t^\pi$. Though the resulting multi-step operator still has $\eta^\pi$ as one fixed point, we show numerically that it is not contractive on the same simple example. These results demonstrate that naively removing the path-dependency might lead to non-convergent multi-step operators. 

\subsection{Backward-view of distributional multi-step learning}

To highlight the difference between distributional and value-based multi-step learning, we discuss the impact that path-dependent distributional TD errors have on the backward-view distributional algorithm. Thus far, distributional back-up targets are expressed in the \emph{forward-view}, i.e., the back-up target at time $t$ is calculated as a function of future transition tuples $(X_s,A_s,R_s)_{s\leq t}$. The forward-view algorithms, unless truncated, wait until the episode finishes to carry out the update, which might be undesirable when the problem is non-episodic or has a very long horizon. 

In the \emph{backward-view}, when encountering a distributional TD error $\Delta_t^\pi$, the algorithm carries out updates for all predictions at time $t' \leq t$ \citep{sutton1998}. To this end, the algorithm needs to maintain additional \emph{partial return traces}, i.e., the partial sum of rewards $G_{t':t}$, in order to calculate the path-dependent TD error $\tilde{\Delta}_t^\pi$. Unlike the value-based state-dependent eligibility traces \citep{sutton1998,van2020expected}, partial return traces are time-dependent. This implies that in an episode of $T$ steps, value-based backward-view algorithms require memory of size $\min(|\mathcal{X}||\mathcal{A}|,\mathcal{O}(T))$ while the distributional algorithms requires $\mathcal{O}(T)$.

In addition to the added memory complexity, the incremental updates of distributional algorithms are also much more complicated due to the path-dependent TD errors. We remark that the path-independent nature of value-based TD errors greatly simplify the value-based backward-view algorithm. For a more detailed discussion, see Appendix~\ref{appendix:backwardview}.

\subsection{Importance sampling for multi-step distributional RL}

In our initial derivation, we arrived at  $\mathcal{R}^{\pi,\mu}$ through the application of importance sampling (IS) in a different way from the value-based setting. We now highlight the subtle differences and caveats. 

For a fixed $n\geq 1$, consider the trace coefficient $c_t=\rho_t\mathbb{I}[t<n]$. The back-up target of the resulting Retrace operator reduces to
$
     \mathbb{E}_{\mu}\left[ \rho_{1:n-1} \cdot \left( \textrm{b}_{G_{0:n-1},\gamma^{n}}\right)_{\#}\eta\left(X_{n},A_{n}^\pi\right)\right]$.
This can be seen as applying IS to the $n$-step prediction $\left( \textrm{b}_{G_{0:n-1},\gamma^{n}}\right)_{\#}\eta\left(X_{n},A_{n}^\pi\right)$. As a caveat, note that an appealing alternative approach is to apply IS to $G_{0:n-1}$, producing the estimate $\left( \textrm{b}_{\rho_{1:n-1}G_{0:n-1},\gamma^{n}}\right)_{\#}\eta\left(X_{n},A_{n}^\pi\right)$. This latter estimate does not properly correct for the off-policy discrepancy between $\pi$ and $\mu$. To see why, note that applying the IS ratio to $G_{0:n-1}$, instead of to the probability of its occurrence, is an artifact of value-based RL because the expected return is linear in $G_{0:t}$ \citep{precup2001off}. In general for distributional RL, one should importance weigh the measures instead of sum of rewards.

\section{Approximate multi-step distributional reinforcement learning algorithm}
\label{sec:parametric}
We now discuss how the distributional Retrace operator combines with parametric distributions, using the construction of the novel Quantile Regression-Retrace algorithm as a practical example. We focus on the quantile representation because it entails the best empirical performance of large-scale distributional RL \citep{dabney2018distributional,dabney2018implicit}. Speficially, we present an application of quantile regression with signed measures, which is interesting in its own right. Below, we start with a brief background on quantile representations \citep{dabney2018distributional}, followed by details on the proposed algorithm.

Consider parametric distributions of the form: $
\frac{1}{m}\sum_{i=1}^m \delta_{z_i}$ for a fixed $m\geq 1$, where $(z_i)_{i=1}^m\in\mathbb{R}$ are a set of parameters indicating the support of the distribution. Let $\probspace_\mathcal{Q}(\mathbb{R})$ denote the family of distribution $
\probspace_\mathcal{Q}(\mathbb{R}) \coloneqq \{\frac{1}{m}\sum_{i=1}^m\delta_{z_i}|z_i\in\mathbb{R}\}$.  We define the projection $\Pi_\mathcal{Q}: \probspace_\infty(\mathbb{R})\rightarrow\probspace_\mathcal{Q}(\mathbb{R})$ as $
   \Pi_\mathcal{Q}\eta=\arg\min_{\nu\in\probspace_\mathcal{Q}(\mathbb{R})} W_1(\eta,\nu)$, which projects any distribution onto the space of representable distributions in the parametric class under the $W_1$ distance. With an abuse of notation, we also let $\Pi_\mathcal{Q}$ denote the component-wise projection when applied to vectors. See \citep{dabney2018distributional,bdr2022} for more details.

\paragraph{Gradient-based learning via quantile regression.} We can use quantile regression \cite{koenker1978regression,koenker2005quantile,koenker2017handbook} to calculate the projection $\Pi_\mathcal{Q}\eta$. Let $F_\eta(z),z\in\mathbb{R}$ denote the CDF of a given distribution $\eta$. Let $F_\eta^{-1}$ be the generalized CDF inverse, we define the $\tau$-th quantile as $F_\eta^{-1}(\tau)$ for $\tau\in[0,1]$. The projection $\Pi_\mathcal{Q}$ is equivalent to computing $z_i=F_\eta^{-1}(\tau_i)$ for $\tau\in (\frac{2i-1}{2m})_{i=1}^m$ \citep{dabney2018distributional}. To learn the $\tau$-th quantile for any $\tau\in[0,1]$, it suffices to solve the quantile regression problem whose optimal solution is $F_\eta^{-1}(\tau)$:
$
    \min_\theta L_\theta^\tau(\eta)\coloneqq \mathbb{E}_{Z\sim \eta}\left[f_\tau(Z-\theta)\right]$ where $f_\tau(u) = u(\tau-\mathbb{I}[u<0])$.
In practice, we carry out the gradient update $\theta\leftarrow\theta-\alpha\nabla_\theta L_\theta^\tau(\eta)$ to find the optimal solution and learn the quantile $\theta\approx F_\eta^{-1}(\tau)$. 

\subsection{Distributional Retrace with quantile representations}
Given an input distribution vector $\eta$, we use the distributional Retrace operator to construct the back-up target $\mathcal{R}^{\pi,\mu}\eta$. Then, we  use the quantile projection to map the back-up target onto the space of representations  $\Pi_\mathcal{Q} \mathcal{R}^{\pi,\mu}\eta$. Overall, we are interested in the recursive update: start with any $\eta_0\in\probspace_\mathcal{Q}(\mathbb{R})^{\mathcal{X}\times{\mathcal{A}}}$, consider the sequence of distributions generated via  $\eta_{k+1}=\Pi_\mathcal{Q} \mathcal{R}^{\pi,\mu} \eta_k$. A direct application of Proposition~\ref{prop:retracecontractive} allows us to characterize the convergence of the sequence, following the approach of \cite{bdr2022}.

\begin{restatable}{theorem}{theoremquantilecontraction}\label{theorem:quantilecontraction} (\textbf{Convergence of quantile distributions})
The projected distributional Retrace operator $\Pi_\mathcal{Q} \mathcal{R}^{\pi,\mu}$ is  $\beta$-contractive under $\bar{W}_\infty$ distance in $\probspace_\mathcal{Q}(\mathbb{R})$. As a result, the above $\eta_k$ converges to a limiting distribution $\eta_\mathcal{R}^\pi$ in $\bar{W}_\infty$, such that
$
     \bar{W}_\infty(\eta_k,\eta_\mathcal{R}^\pi)\leq(\beta)^k \bar{W}_\infty(\eta_0,\eta_\mathcal{R}^\pi)
$. Further, the quality of the fixed point is characterized as $\bar{W}_\infty(\eta_\mathcal{R}^\pi,\eta^\pi)\leq (1-\beta)^{-1}\bar{W}_\infty(\Pi_\mathcal{Q}\eta^\pi,\eta^\pi)$.
\end{restatable}

Thanks to the faster contraction rate $\beta\leq\gamma$,
the advantage of the projected operator $\Pi_\mathcal{Q} \mathcal{R}^{\pi,\mu}$ is two-fold: (1) the operator often contracts faster to the limiting distribution $\eta_\mathcal{R}^\pi$ than the one-step operator $\mathcal{T}^\pi$ contracts to its own limiting distribution $\eta_{\mathcal{T}^\pi}$ \citep{dabney2018distributional}; (2) the limiting distribution $\eta_\mathcal{R}^\pi$ also enjoys a better approximation bound to the target distribution. We verify such results in Section~\ref{sec:experiments}.

\subsection{Quantile Regression-Retrace: distributional Retrace with quantile regression}
Below, we use $z_i(x,a)$ to represent the $i$-th quantile of the distribution at $(x,a)$. Overall, we have a tabular quantile representation $\eta_z(x,a)=\frac{1}{m}\sum_{i=1}^m \delta_{z_i(x,a)},\forall (x,a)\in\mathcal{X}\times\mathcal{A}$, where we use the notation $\eta_z$ to stress the distribution's dependency on parameter $z_i(x,a)$. For any given bootstrapping distribution vector $\eta\in\probspace_\infty(\mathbb{R})^{\mathcal{X}\times\mathcal{A}}$, in order to approximate the projected back-up target $ \Pi_\mathcal{Q} \mathcal{R}^{\pi,\mu}\eta$ with the parameterized quantile distribution $\eta_z$, we solve the set of quantile regression problems for all $1\leq i\leq m, (x,a)\in\mathcal{X}\times\mathcal{A}$,
\begin{align*}
  \min_{z_i(x,a)}  L_{z_i(x,a)}^{\tau_i}\left(\mathcal{R}^{\pi,\mu}\eta(x,a)\right),\ \text{where}\ \tau_i=(2i-1)/2m \, .
\end{align*}
For any fixed $(x,a,i)$, to solve the quantile regression problem, we apply gradient descent on $z_i(x,a)$. In practice, with one sampled trajectory $(X_s,A_s,R_s)_{s=0}^\infty\sim \mu$, the aim is to construct an unbiased stochastic gradient estimate of the QR loss $ L_{z_i(x,a)}^{\tau_i}\left(\mathcal{R}^{\pi,\mu}\eta(x,a)\right)$. Below, let $\textrm{b}_t=\textrm{b}_{G_{0:t-1},\gamma^{t}}$ for simplicity. We start with a stochastic estimate $\hat{L}_{z_i(x,a)}^{\tau_i}(\mathcal{R}^{\pi,\mu}\eta(x,a))$ for the QR loss,
 \begin{align*}
   L_{z_i(x,a)}^{\tau_i} \left(\eta(x,a)\right) + \sum_{t=0}^\infty c_{1:t} \left( L_{z_i(x,a)}^{\tau_i}\left(\left(\textrm{b}_{t+1}\right)_{\#}\eta\left(X_{t+1},A_{t+1}^\pi\right)\right) - L_{z_i(x,a)}^{\tau_i}\left(\left(\textrm{b}_t\right)_{\#}\eta(X_t,A_t)\right)\right).
 \end{align*}
 Since $\hat{L}_{z_i(x,a)}^{\tau_i}(\mathcal{R}^{\pi,\mu}\eta(x,a))$ is differentiable with $z_i(x,a)$, we use 
$\nabla_{z_i(x,a)}\hat{L}_{z_i(x,a)}^{\tau_i}(\mathcal{R}^{\pi,\mu}\eta(x,a))$ as the stochastic gradient estimate. This gradient estimate is unbiased under mild conditions.
\begin{restatable}{lemma}{lemmaunbiasedqr}\label{lemma:unbiasedqr} (\textbf{Unbiased stochastic QR loss gradient estimate}) Assume that the trajectory terminates within $H<\infty$ steps almost surely, then we have $\mathbb{E}_\mu[\hat{L}_{z_i(x,a)}^{\tau_i}\left(\mathcal{R}^{\pi,\mu}\eta(x,a)\right)]=L_{z_i(x,a)}^{\tau_i}\left(\mathcal{R}^{\pi,\mu}\eta(x,a)\right)$ and $\mathbb{E}_\mu[\nabla_{z_i(x,a)}\hat{L}_{z_i(x,a)}^{\tau_i}\left(\mathcal{R}^{\pi,\mu}\eta(x,a)\right)]=\nabla_{z_i(x,a)}L_{z_i(x,a)}^{\tau_i}\left(\mathcal{R}^{\pi,\mu}\eta(x,a)\right)$.
\end{restatable}

The above stochastic estimate bypasses the challenge that the QR loss is only defined against distributions, whereas sampled back-up targets $\hat{R}^{\pi,\mu}\eta(x,a)=\eta(x,a)+\sum_{t=0}^\infty c_{1:t}\tilde{\Delta}_{0:t}^\pi$ are signed measures in general. In Quantile Regression-Retrace, we use $\eta_z$ itself as the bootstrapping distribution, such that the algorithm approximates the fixed point iteration $\eta_z\leftarrow \Pi_\mathcal{Q}\mathcal{R}^{\pi,\mu}\eta_z$. Concretely, we carry out the following sample-based update
\begin{align*}
    z_i(x,a)\leftarrow z_i(x,a) - \alpha \nabla_{z_i(x,a)} \hat{L}_{z_i(x,a)}^{\tau_i}\left(\mathcal{R}^{\pi,\mu}\eta_z(x,a)\right),\ \text{for}\ \forall\ 1\leq i\leq m, (x,a)\in\mathcal{X}\times\mathcal{A}.
\end{align*}

\subsection{Deep reinforcement learning: QR-DQN-Retrace}

We introduce a deep RL implementation of the Quantile Regression-Retrace: QR-DQN-Retrace, where the parametric representation is combined with function approximations \citep{bellemare2017cramer,dabney2018distributional,dabney2018implicit}. The base agent QR-DQN \citep{bellemare2017cramer} parameterizes the quantile locations $z_i(x,a;w)$ with the output of a neural network with weights $w$. Let $\eta(x,a;w)=\frac{1}{m}\sum_{i=1}^m \delta_{z_i(x,a;w)}$ denote the parameterized distribution. QR-DQN-Retrace  updates its parameters by stochastic gradient descent on the estimated QR loss, averaged across all $m$ quantile levels $
    w\leftarrow w - \alpha \frac{1}{m}\sum_{i=1}^m \nabla_w \hat{L}_{z_i(x,a;w)}^{\tau_i}\left(\mathcal{R}^{\pi,\mu}\eta(x,a;w)\right)$. In practice, the update is further averaged over state-action pairs sampled from a replay buffer.

\begin{figure}[t]
    \centering
    \subfigure[Off-policyness $\epsilon$ ]{\includegraphics[keepaspectratio,width=.24\textwidth]{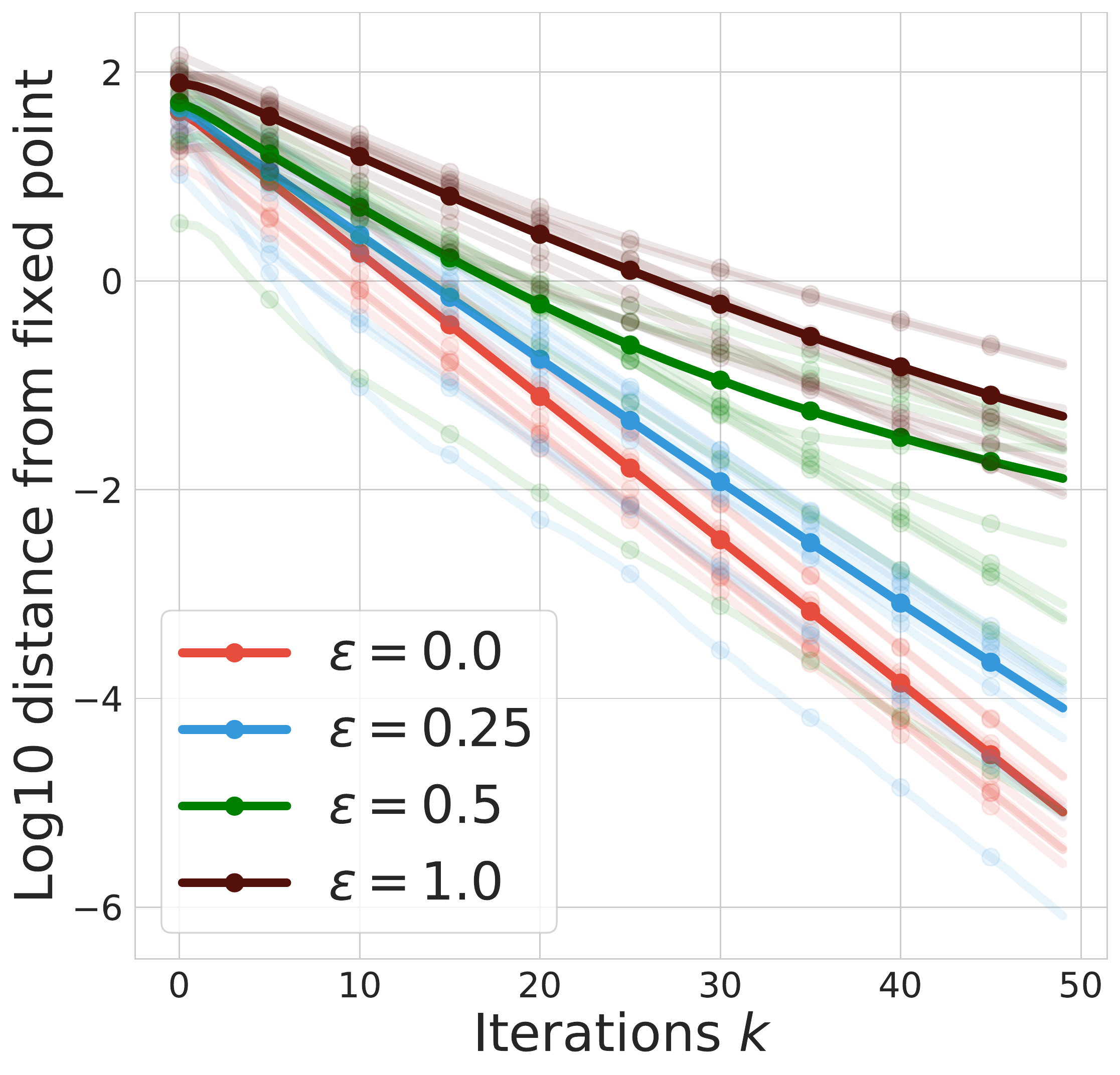}}
    \subfigure[Trace coefficient $c_t$]{\includegraphics[keepaspectratio,width=.24\textwidth]{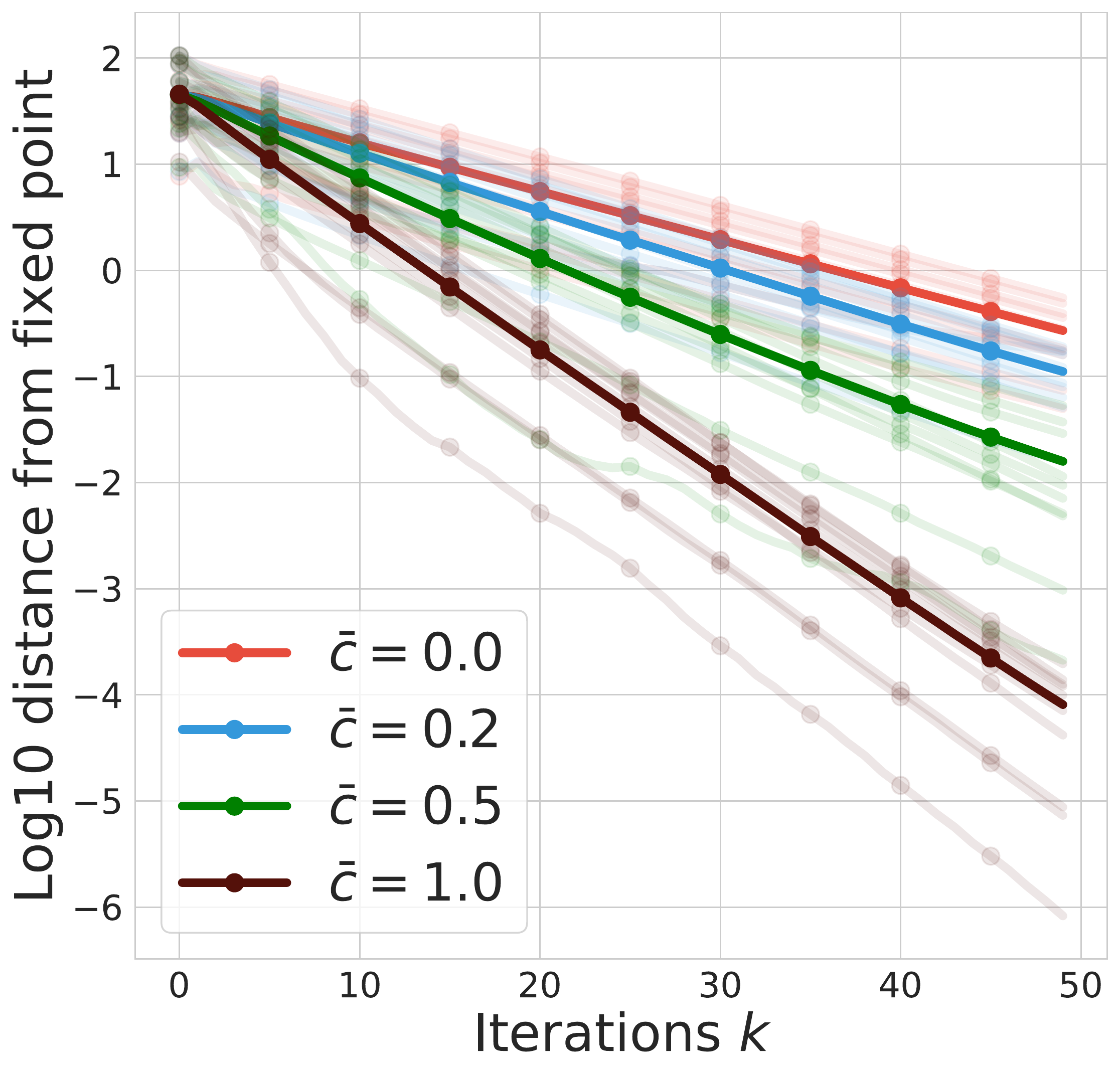}}
    \subfigure[Fixed point quality]{\includegraphics[keepaspectratio,width=.24\textwidth]{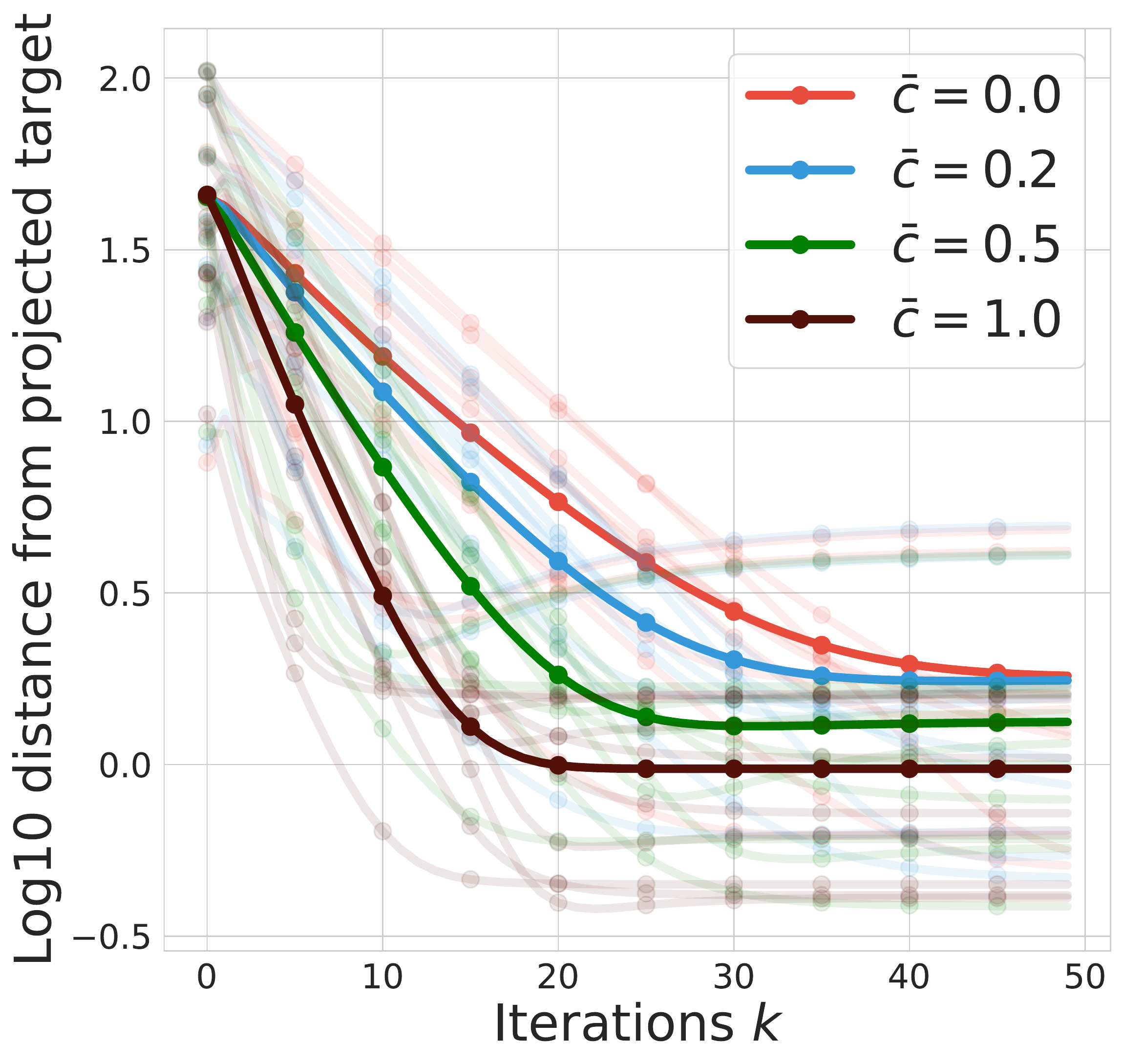}}
    \subfigure[Uncorrected ]{\includegraphics[keepaspectratio,width=.24\textwidth]{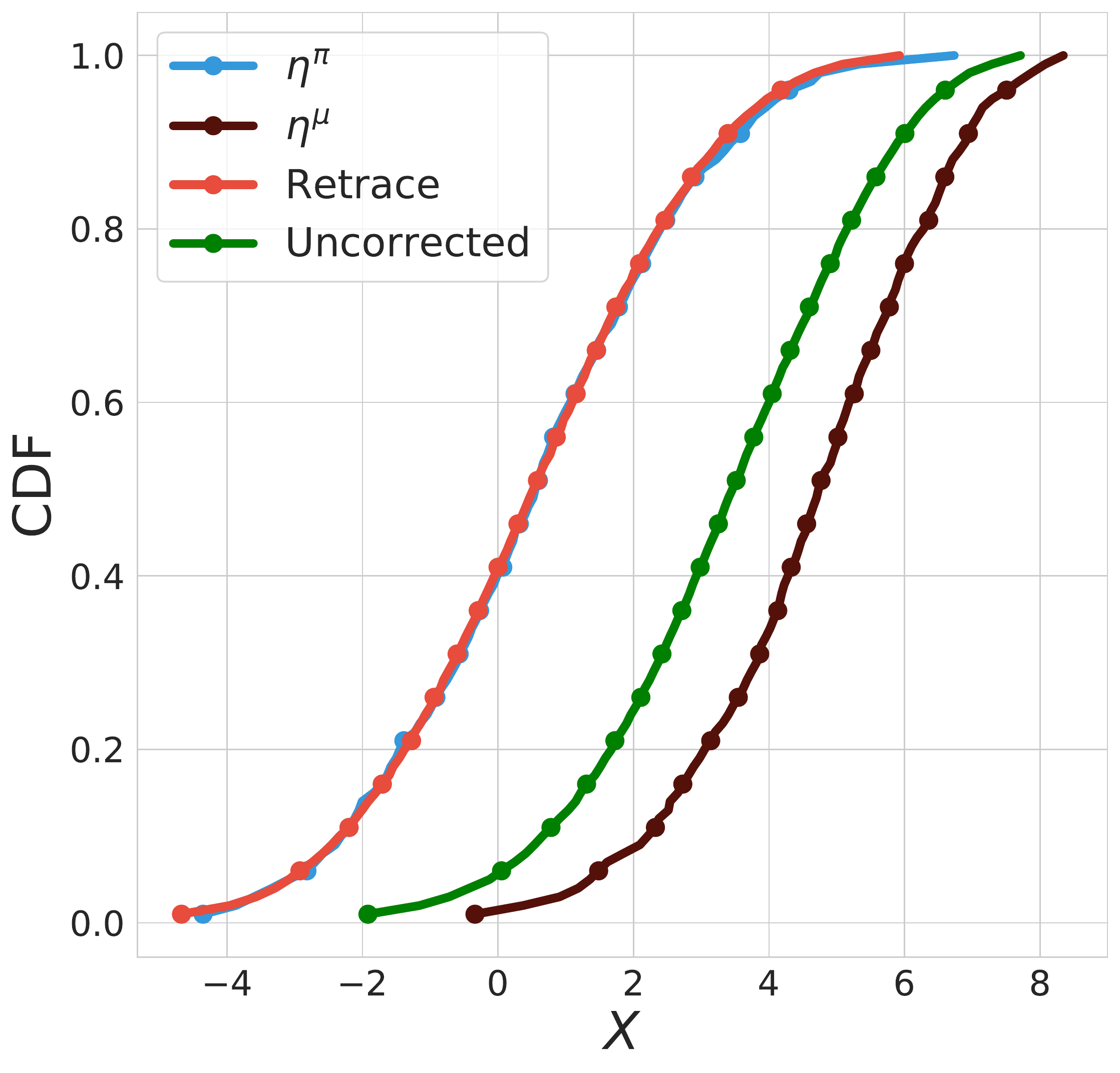}}
    \caption{Tabular experiments to illustrate properties of the distributional Retrace operator: we show average results across $10$ randomly sampled MDPs. (a) Contraction rate vs. off-policyness; (b) Contraction rate vs. trace coefficient $c_t=\min(\rho_t,\bar{c})$; (c) Fixed point quality vs. trace coefficient $c_t$; (d) The uncorrected operator introduces bias to the fixed point while Retrace is unbiased.}
    \label{fig:tabular}
\end{figure}

\section{Discussions}

\paragraph{Categorical representations.} The categorical representation is another commonly used class of parameterized distributions in prior literature \citep{bellemare2017cramer,rowland2018analysis,rowland2019statistics,bdr2022}. We obtain contractive guarantees for the categorical representation similar to Theorem~\ref{theorem:quantilecontraction}. As with QR, this leads both to improved fixed-point approximations and faster convergence. Further, this leads to a deep RL algorithm C51-Retrace.
The actor-critic Reactor agent \citep{gruslys2017reactor} uses C51-Retrace as a critic training algorithm, although without explicit consideration or analysis of the associated distributional operator. 
See Appendix~\ref{appendix:categorical} for details. We empirically evaluate the stand-alone improvements of C51-Retrace over C51 in Section~\ref{sec:experiments}.

\paragraph{Uncorrected methods.} The uncorrected methods do not correct for the off-policyness and hence obtain a biased fixed point \citep{hessel2018rainbow,kapturowski2018recurrent,kozuno2021revisiting}. The Rainbow agent \citep{hessel2018rainbow} combined $n$-step uncorrected learning with C51, effectively implementing a distributional operator whose fixed point differs from $\eta^\pi$.

\paragraph{On-policy distributional TD($\lambda$).} Nam et al. \cite{nam2021gmac} propose SR($\lambda$), a distributional version of on-policy TD$(\lambda$) \citep{sutton1988learning}. In operator form, this can be viewed as a special case of Equation~\eqref{eq:dist-retrace-operator} with $\mu=\pi$, $c_t = \lambda$; \cite{nam2021gmac} also introduce a sample-replacement technique for more efficient implementation.

\section{Experiments}
\label{sec:experiments}

We carry out a number of experiments to validate the theoretical insights and empirical improvements.

\subsection{Illustration of distributional Retrace properties on tabular MDPs}

We verify a few important properties of the distributional Retrace operator on a tabular MDP. The results corroborate the theoretical results from previous sections. Throughout, we use quantile representations with $m=100$ atoms; we obtain similar results for categorical representations. See Appendix~\ref{appendix:experiments} for details on the experiment setup. Let $\eta_0$ be the initial distribution, we carry out dynamic programming with $\mathcal{R}^{\pi,\mu}$ and denote $\eta_k=(\mathcal{R}^{\pi,\mu})^k\eta_0$ as the $k$\textsuperscript{th} distribution iterate. 
\paragraph{Impact of off-policyness.} We control the level of off-policyness by setting the behavior policy $\mu$ to be a uniform policy and the target policy to $\pi=(1-\epsilon)\mu+\epsilon \pi_d$ where $\pi_d$ is a fixed deterministic policy. Moving from $\epsilon=0$ to $\epsilon=1$, we transition from on-policy to very off-policy. We use $L_p(\eta_k,\eta_\mathcal{R}^\pi)$ to measure the contraction rate to the fixed point. Figure~\ref{fig:tabular} shows that as the behavior  becomes more off-policy, the contraction slows down, degrading the efficiency of multi-step learning.

\paragraph{Impact of trace coefficient $c_t$.} Throughout, we set $c_t=\min(\rho_t,\bar{c})$ with $\bar{c}$ to control the effective trace length. With a fixed level of off-policyness $\epsilon=0.5$, Figure~\ref{fig:tabular}(b) shows that increasing $\bar{c}$ speeds up the contraction to the fixed point as predicted by Proposition~\ref{prop:retracecontractive}.

\paragraph{Quality of fixed point.} We next examine how the quality of the fixed point is impacted by $\bar{c}$, by measuring $L_p(\eta_k,\Pi_\mathcal{Q}\eta^\pi)$ as a proxy to $L_p(\eta_k,\eta^\pi)$. As $k$ increases the error flattens, at which point we take the converged value to be $L_p(\eta_\mathcal{R}^\pi,\Pi_\mathcal{Q}\eta^\pi)$ which measures the fixed point quality. Figure~\ref{fig:tabular}(c) shows when $\bar{c}$ increases, the fixed point quality improves, in line with the Theorem~\ref{theorem:quantilecontraction}. This phenomenon does not arise in \emph{tabular} non-distributional reinforcement learning, although related phenomena do occur when using function approximation techniques.

\paragraph{Bias of uncorrected methods.} Finally, we illustrate a critical difference between Retrace and uncorrected $n$-step methods \citep{hessel2018rainbow}: the bias to the fixed point. Figure~\ref{fig:tabular}(d) shows that uncorrected $n$-step arrives at a fixed point in between $\eta^\pi$ and $\eta^\mu$, showing an obvious bias from $\eta^\pi$. 

\subsection{Deep reinforcement learning}

 We consider the control setting where the target policy $\pi$ is the greedy policy with respect to the Q-function induced by the parameterized distribution. Because the training data is sampled from a replay, the behavior policy $\mu$ is $\epsilon$-greedy with respect to Q-functions induced by previous copies of the parameterized distribution. We evaluate the performance of deep RL agents on 57 Atari games \citep{bellemare2013arcade}. To ensure fair comparison across games, we compute the human normalized scores for each agent, and compare their evaluated mean and median scores across all 57 games during training.

\paragraph{Deep RL agents.} The multi-step agents adopt exactly the same hyperparameters as the baseline agents. The only difference is the back-up target. For completeness of results, we show the combination of Retrace with both C51 and QR-DQN. For QR-DQN, we use the Huber loss for quantile regression, which is a thresholded variant of the QR loss \citep{dabney2018distributional}. Throughout, we use $c_t=\lambda\min(\rho_t,\bar{c})$ with $\bar{c}=1$ as in \citep{munos2016safe}. See Appendix~\ref{appendix:experiments} for details. In practice, sampled trajectories are truncated at length $n$. We also adapt Retrace to the $n$-step case, see Appendix~\ref{appendix:nstep}.

\paragraph{Results.} Figure~\ref{fig:deeprl} compares one-step baseline, Retrace and uncorrected $n$-step \citep{hessel2018rainbow}. For C51, both multi-step methods clearly improve the median performance over the one-step baseline. Retrace slightly outperforms uncorrected $n$-step towards the end of learning. For QR-DQN, all multi-step algorithms achieve clear performance gains. Retrace significantly outperforms the uncorrected $n$-step with the mean performance, while obtaining similar results on the median performance. Overall, distributional Retrace achieves a clear improvement over the one-step baselines. The uncorrected $n$-step method typically takes off faster than Retrace but may to slightly worse performance. 

Finally, note that in the value-based setting, uncorrected methods are generally more high-performing than Retrace, potentially due to a favorable trade-off between contraction rate and fixed-point bias \citep{rowland2019adaptive}. Our results add to the benefits of off-policy corrections in the control setting.

\begin{figure}[t]
    \centering
    \subfigure[C51 ]{\includegraphics[keepaspectratio,width=.49\textwidth]{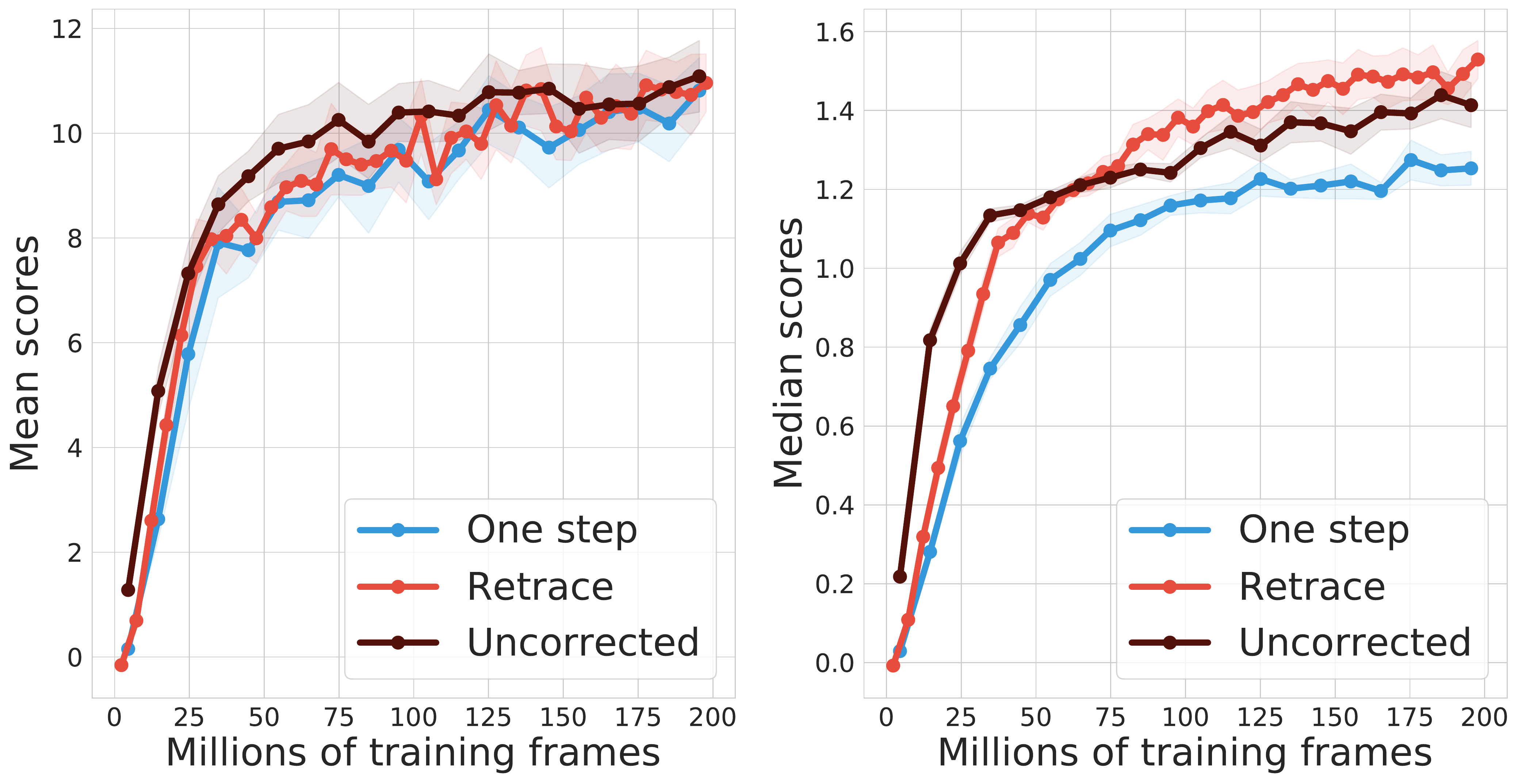}}
    \subfigure[QR-DQN ]{\includegraphics[keepaspectratio,width=.49\textwidth]{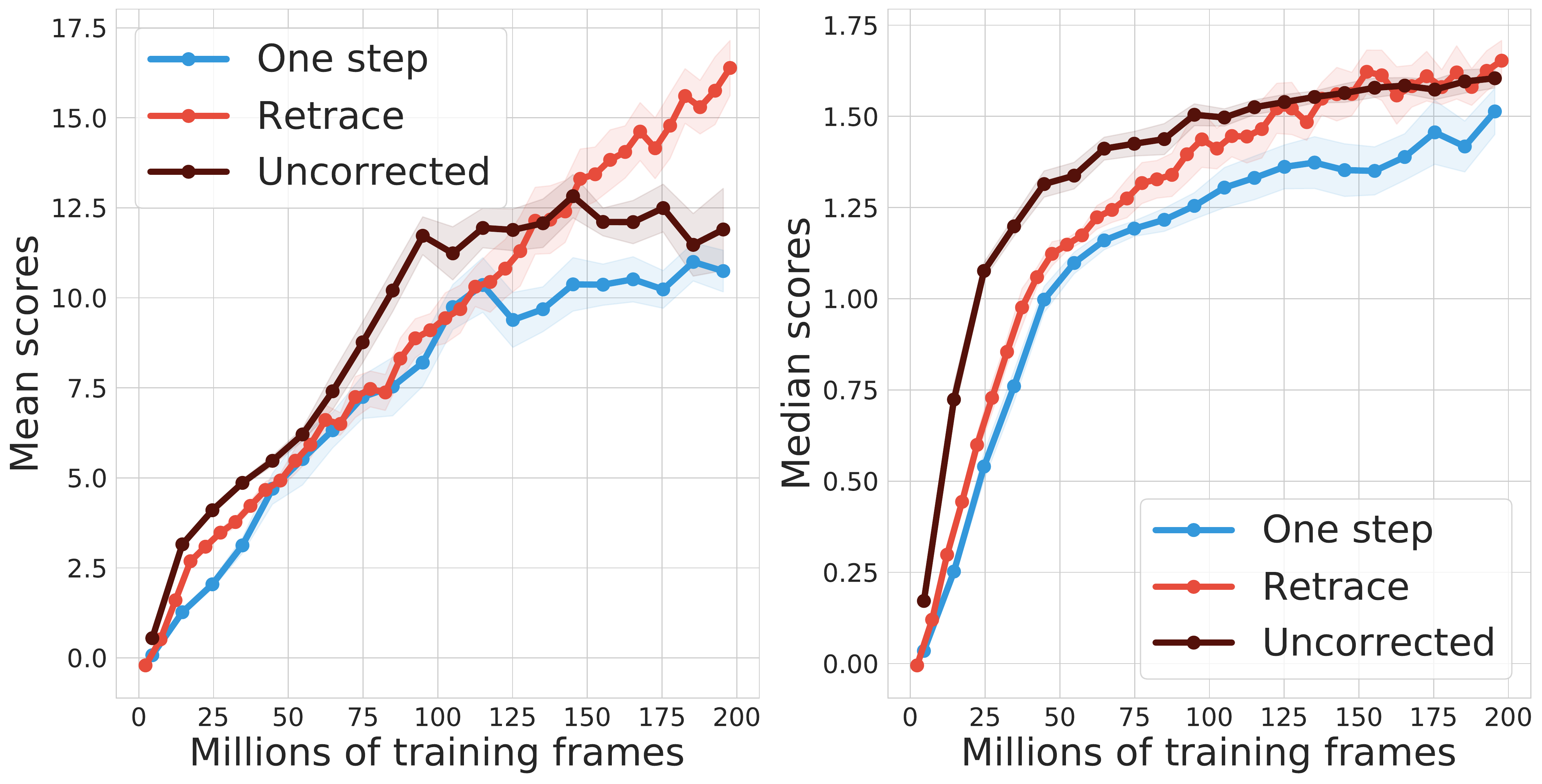}}
    \caption{Deep RL experiments on Atari-57 games for (a) C51 and (b) QR-DQN. We compare the one-step baseline agent against the multi-step variants (Retrace and uncorrected $n$-step). For all multi-step variants, we use $n=3$. For each agent, we calculate the mean and median performance across all games, and we plot the $\text{mean}\pm\text{standard\ error}$ across $3$ seeds. In almost all settings, multi-step variants provide clear advantage over the one-step baseline algorithm.}
    \label{fig:deeprl}
\end{figure}

\section{Conclusion}

We have identified a number of fundamental conceptual differences between value-based and distributional RL in multi-step settings. Central to such differences is the novel notion of path-dependent distributional TD error, which naturally arises from the multi-step distributional RL problem. Building on this understanding, we have developed the first principled multi-step off-policy distributional operator Retrace. We have also developed an approximate distributional RL algorithm, Quantile Regression-Retrace, which makes distributional Retrace highly competitive in both tabular and high-dimensional setups. This paper also opens up a several avenues for future research, such as the interaction between multi-step distributional RL and signed measures, and the convergence theory of stochastic approximations for multi-step distributional RL.

\bibliographystyle{unsrt}
\bibliography{your_bib_file}

\newpage
\appendix

\section*{\centering The Nature of Temporal Difference Errors in \\ Multi-step Distributional Reinforcement Learning:\\ Appendices}

\section{Extension of distributional Retrace to $n$-step truncated trajectories}
\label{appendix:nstep}
The $n$-step truncated version of distributional Retrace is defined as 
\begin{align*}
    \mathcal{R}_n^{\pi,\mu}\eta(x,a) = \eta(x,a) + \mathbb{E}_\mu\left[\sum_{t=0}^n c_{1:t}\tilde{\Delta}_{0:t}^\pi\right],
\end{align*}
which sums the path-dependent distributional TD errors up to time $n$. Compared to the original definition of distributional Retrace, this $n$-step operator is more practical to implement. This operator enjoys all the theoretical properties of the original distributional Retrace, with a slight difference on the contraction rate. Intuitively, the operator bootstraps with at most $n$ steps, which limits the effective horizon of the operator to be $\leq n$.
It is straightforward to show that the operator is $\beta_n$-contractive under $\bar{W}_p$ with $\beta_n\in(\beta,\gamma]$. As $n\rightarrow\infty$, $\beta_n\rightarrow\beta$.

\section{Distance metrics}
\label{appendix:metric}

We provide a brief review on the distance metrics used in this work. We refer readers to \citep{bdr2022} for a complete background.

\subsection{Wasserstein distance}

Let $\eta_1,\eta_2\in\probspace_\infty(\mathbb{R})$ be two distribution measures. Let $F_\eta$ be the CDF of $\eta$. The $p$-Wasserstein distance can be computed as
\begin{align*}
    W_p(\eta_1,\eta_2) \coloneqq \left(\int_{[0,1]} |F_{\eta_1}^{-1}(z) - F_{\eta_2}^{-1}(z)|^p dz\right)^{1/p}.
\end{align*}
Note that the above definition is equivalent to the more traditional definition based on optimal transport; indeed, $F_{\eta_i}^{-1}(z),z\sim \text{Uniform}(0,1),i\in\{1,2\}$ can be understood as the optimal coupling between the two distributions. The above definition is a proper distance metric if $p\geq 1$.

For any distribution vector $\eta_1,\eta_2\in\probspace_\infty(\mathbb{R})^{\mathcal{X}\times\mathcal{A}}$, we can define the supremum $p$-Wasserstein distance as 
\begin{align*}
    \bar{W}_p(\eta_1,\eta_2) \coloneqq \max_{x,a} W_p(\eta_1(x,a),\eta_2(x,a)).    
\end{align*}

\subsection{$L_p$ distance}

Let $\eta_1,\eta_2\in\probspace_\infty(\mathbb{R})$ be two distribution measures. Let $F_\eta$ be the CDF of $\eta$. The $L_p$ distance is defined as 
\begin{align*}
    L_p(\eta_1,\eta_2) \coloneqq \left(\int_\mathbb{R} \left|F_{\eta_1}(z) - F_{\eta_2}(z)\right|^p dz \right)^{1/p}.
\end{align*}
The above definition is a proper distance metric when $p\geq 1$.

For any distribution vector $\eta_1,\eta_2\in\probspace_\infty(\mathbb{R})^{\mathcal{X}\times\mathcal{A}}$ or signed measure vector
$\eta_1,\eta_2\in\mathcal{M}(\mathbb{R})^{\mathcal{X}\times\mathcal{A}}$, we can define the supremum Cram\'er-$p$ distance as 
\begin{align*}
    \bar{L}_p(\eta_1,\eta_2) \coloneqq \max_{x,a} L_p(\eta_1(x,a),\eta_2(x,a)).    
\end{align*}

\section{Numerically non-convergent behavior of alternative multi-step operators}
\label{appendix:failure}

We consider another alternative definition of path-independent alternative to the path-dependent TD error $\gamma^t \Delta_t^\pi$. The primary motivation for such a path-dependent TD error is that the discounted value-based TD error takes the form $\tilde{\delta}_t^\pi=\gamma^t \delta_t^\pi$. The resulting multi-step operator is
\begin{align*}
    \tilde{\mathcal{R}}^{\pi,\mu}\eta(x,a) = \eta(x,a) +\mathbb{E}_\mu\left[\sum_{t=0}^\infty c_{1:t}\gamma^t \Delta_t^\pi\right].
\end{align*}

With the same toy example as in the paper: an one-state one-action MDP with a deterministic reward $R_t=1$ and discount factor $\gamma=0.5$. The target distribution $\eta^\pi$ is a Dirac distribution centering at $2$. Let $\eta_k=(\mathcal{R})^k\eta_0$ be the $k$-th distribution iterate by applying the operator $\mathcal{R}\in\{\mathcal{R}^{\pi,\mu},\tilde{\mathcal{R}}^{\pi,\mu},\tilde{\mathcal{R}}^{\pi,\mu},\mathcal{T}^\pi\}$, we show the $L_p$ distance between the iterates and $\eta^\pi$ in Figure~\ref{fig:non-convergence}. It is clear that alternative multi-step operators do not converge to the correct fixed point.

\begin{figure}
    \centering
    \subfigure[Full results for all operators ]{\includegraphics[keepaspectratio,width=.48\textwidth]{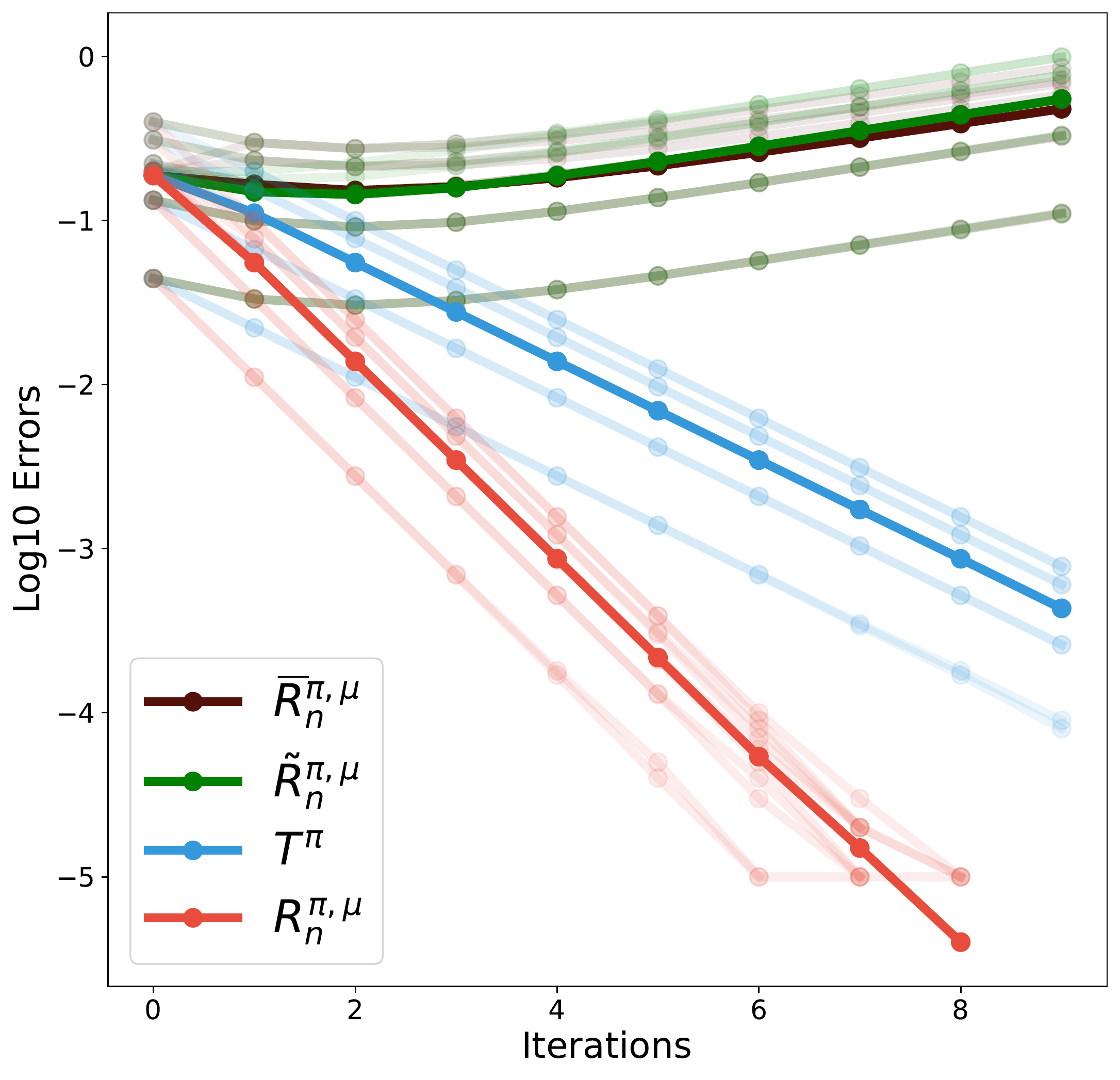}}
    \subfigure[Comparing two alternatives ]{\includegraphics[keepaspectratio,width=.48\textwidth]{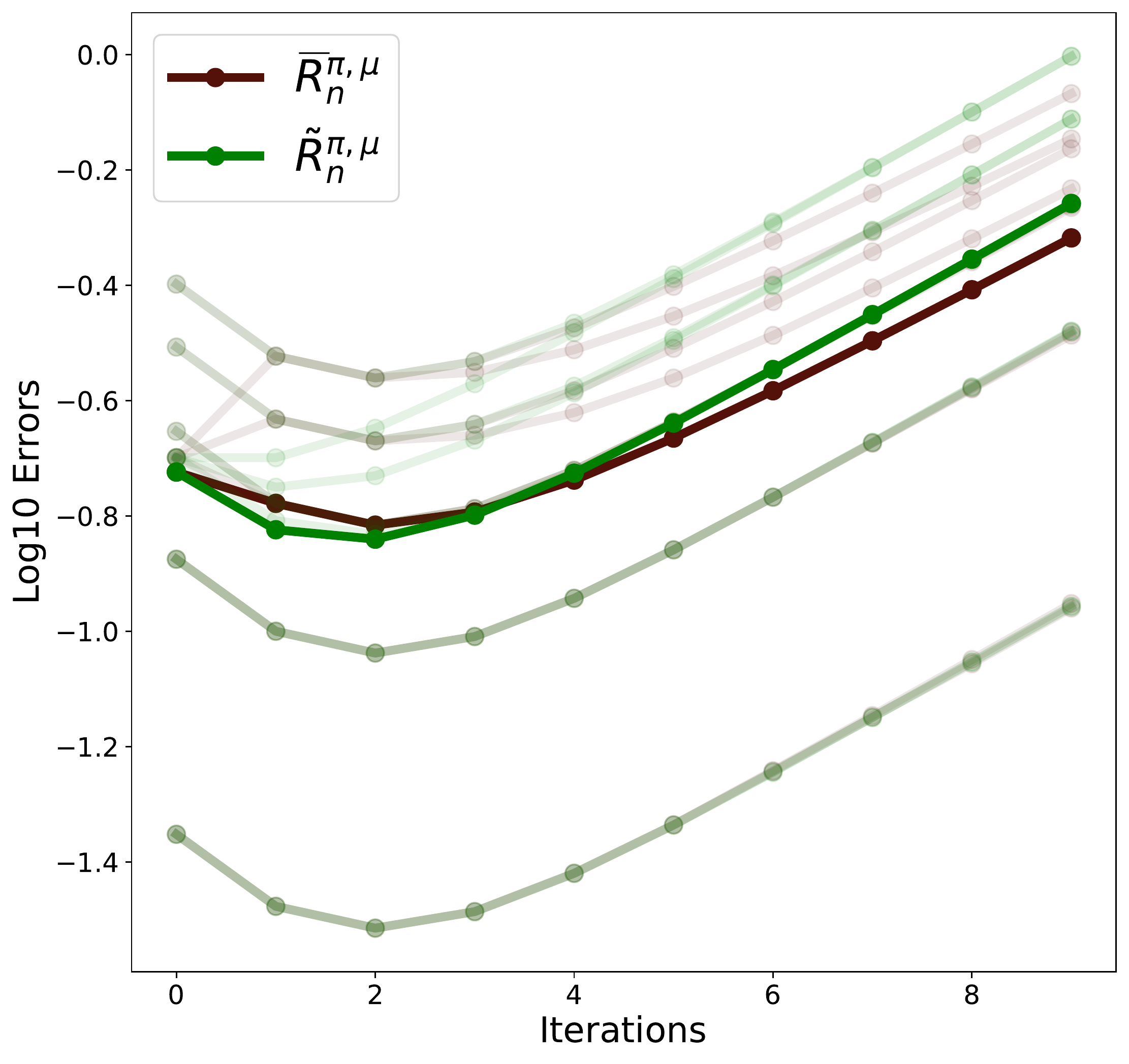}}
    \caption{Illustration of non-convergent behavior of alternative multi-step operators: for both plots, we show the mean and per-run results across $10$ different initial Dirac distributions $\eta_0$. (a) the full comparison between all operators. Two alternative operators do not converge while one-step Bellman operator and distributional Retrace both converge; (b) we zoom in on the difference between the two alternative operators. }
    \label{fig:non-convergence}
\end{figure}

\section{Backward-view algorithm for multi-step distributional RL}
\label{appendix:backwardview}

We now describe a backward-view algorithm for multi-step distributional RL with quantile representations. For simplicity, we consider the on-policy case $\pi=\mu$ and $c_t=\lambda$. To implement $\mathcal{R}^{\pi,\mu}$ in the backward-view, at each time step $t$ and a past time step $t'\leq t$, the algorithm needs to maintain two novel traces distinct from the classic eligibility traces \citep{sutton1998}: (1) partial return traces $G_{t':t}$, which correspond to the partial sum of rewards between two time steps $t'\leq t$; (2) modified eligibility traces, defined as $e_{t',t}\coloneqq\lambda^{t-t'}$, which measures the trace decay between two time steps $t'\leq t$. At a new time step $t+1$, the new traces are computed recursively: $G_{t':t+1}=R_{t+1}+\gamma G_{t',t}, e_{t',t+1}= \lambda e_{t',t}$.

We assume the algorithm maintains a table of quantile distributions with $m$ atoms: $\eta(x,a)=\frac{1}{m}\sum_{i=1}^m \delta_{z_i(x,a)},\forall (x,a)\in\mathcal{X}\times\mathcal{A}$. For any fixed $(x,a)$, define $T_t(x,a)\coloneqq \{s| X_{s}=x,A_{s}=a,0\leq s\leq t\}$ be the set of time steps before time $t$ at which $(x,a)$ is visited. Now, upon arriving at $X_{t+1}$, we observe the TD error $\Delta_t^\pi$. Recall that $L_\theta^\tau(\eta)$ denote the QR loss of parameter $\theta$ at quantile level $\tau$ and against the distrbution $\eta$. To more conveniently describe the update, we define the QR loss against the path-dependent TD error 
\begin{align*}
    \left(\textrm{b}_{G_{s:t-1},\gamma^{t-s}}\right)_\# \tilde{\Delta}_{0:t}^\pi=\left(\textrm{b}_{G_{s:t},\gamma^{t+1-s}}\right)_\# \eta(X_{t+1},A_{t+1}^\pi) - \left(\textrm{b}_{G_{s:t-1},\gamma^{t-s}}\right)_\# \eta(X_t,A_t)
\end{align*} 
as the difference of the QR losses against the individual distributions,
\begin{align*}
    L_\theta^\tau \left(\left( \textrm{b}_{G_{s:t-1},\gamma^{t-s}}\right)_\# \tilde{\Delta}_{0:t}^\pi \right) \coloneqq L_\theta^\tau\left(\left(\textrm{b}_{G_{s:t},\gamma^{t+1-s}}\right)_\# \eta(X_{t+1},A_{t+1}^\pi) \right) - L_\theta^\tau\left(  \left(\textrm{b}_{G_{s:t-1},\gamma^{t-s}}\right)_\# \eta(X_t,A_t)\right).
\end{align*}
Note that the QR loss can be computed using the transition data we have seen so far. We now perform the a gradient update for all entries in the table $(x,a)\in\mathcal{X}\times\mathcal{A}$ and $1\leq i\leq m$ (in practice, we update entries that correspond to visited state-action pairs):
\begin{align*}
    z_i(x,a) \leftarrow z_i(x,a) - \alpha \sum_{s\in T_t(x,a)} e_{s,t} \nabla_{z_i(x,a)}  L_\theta^{\tau_i} \left(\left( \textrm{b}_{G_{s:t-1},\gamma^{t-s}}\right)_\# \tilde{\Delta}_{0:t}^\pi \right),
\end{align*}
where $\tau_i=\frac{2i-1}{2m}$. For any fixed $(x,a)$, the above algorithm effectively aggregates updates from time steps $s\in T_t(x,a)$ at which $(x,a)$ is visited. 

\subsection{Simplifications for value-based RL}

We now discuss how the path-independent value-based TD errors greatly simplify the value-based backward-view algorithm. Following the above notations, assume the algorithm maintains a table of Q-function $Q(x,a)$, we can construct incremental backward-view update for all $(x,a)\in\mathcal{X}\times\mathcal{A}$ as follows, by replacing the path-dependent distributional TD error $\tilde{\Delta}_{0:t}^\pi$ by the discounted TD error $\tilde{\delta}_t^\pi$
\begin{align*}
    Q(x,a)\leftarrow Q(x,a) - \alpha \sum_{s\in T_t(x,a)} e_{s,t} \tilde{\delta}_t^\pi.
\end{align*}
Since $\delta_t^\pi$ does not depend on the past rewards and is state-action dependent, we can simplify the summation over $s\in T_t(x,a)$ by defining the state-depedent eligibility traces \citep{sutton1998} as a replacement to $e_{s,t}$,
\begin{align*}
    \tilde{e}(x,a) \leftarrow \gamma\lambda\tilde{e}(x,a) + \mathbb{I}[X_t=x, A_t=a].
\end{align*}
As a result, the above update reduces to
\begin{align*}
    Q(x,a) \leftarrow Q(x,a) - \alpha \tilde{e}(x,a) \delta_t^\pi,
\end{align*}
which recovers the classic backward-view update.

\subsection{Non-equivalence of forward-view and backward-view algorithms}
In value-based RL, forward-view and backward-view algorithms are equivalent given that the trajectory does not visit the same state twice \citep{sutton1998}. However, such an equivalence does not generally hold in distributional RL. Indeed, consider the following counterexample in the case of the quantile representation.

Consider a three-step MDP with deterministic transition $x_1\rightarrow x_2\rightarrow x_3$. There is no action and no reward on the transition. The state $x_3$ is terminal with a deterministic terminal value $r_3=1$. We consider $m=1$ atom and let the quantile parameters be $\theta_1=0$ and $\theta_2=1$ at states $x_1,x_2$ respectively. In this case, the quantile representation learns the median of the target distribution with $\tau=0.5$. 

Now, we consider the update at $\theta_1$ with both forward-view and backward-view implementation of the two-step Bellman operator $\mathcal{T}_2^\pi\eta(x) = \mathbb{E}\left[(\textrm{b}_{0,\gamma^2})_\# \eta(X_2,\pi(X_2))|X_0=x\right]$, which can be obtained from distributional Retrace by setting $c_t=\rho_t$. The target distribution at $x_1$ is a Dirac distribution centering at $\gamma^2$.

\paragraph{Forward-view update.}
Below, we use $\delta_x$ to denote a Dirac distribution at $x$.
In the forward-view, the back-up distribution is 
\begin{align*}
\mathbb{E}\left[(\textrm{b}_{0,\gamma^2})_\# \eta(X_2,\pi(X_2))\right] = \delta_{\gamma^2}.  
\end{align*}
The gradient update to $\theta_1$ is thus
\begin{align*}
    \theta_1^\text{(fwd)} = \theta_1 - \alpha \nabla_{\theta_1} L_{\theta_1}^{0.5}\left(\delta_{\gamma^2}\right) = \theta_1 + \alpha \left(0.5 - \mathbb{I}\left[\gamma^2 < \theta_1\right]\right).
\end{align*}

\paragraph{Backward-view update.} To implement the backward-view update, we make clear of the two path-dependent distributional TD errors at two consecutive time steps
\begin{align*}
    \tilde{\Delta}_0^\pi = \delta_{\gamma} - \delta_0, \ \ \tilde{\Delta}_1^\pi = (\textrm{b}_{0,\gamma})_\#\left(\delta_{\gamma \theta_2}-\delta_{\theta_1}\right) = \delta_{\gamma^2} - \delta_\gamma
\end{align*}
The update consists of two steps:
\begin{align*}
    \theta_1' &= \theta_1 - \alpha \nabla_{\theta_1} L_{\theta_1}^{0.5}\left(\delta_\gamma\right) = \theta_1 + \alpha \left(0.5 - \mathbb{I}\left[\gamma < \theta_1\right]\right) ,\\
    \theta_1^\text{(bwd)} &= \theta_1' - \alpha \left(\nabla_{\theta_1'} L_{\theta_1'}^{0.5}\left(\delta_\gamma^2\right) - \nabla_{\theta_1'} L_{\theta_1'}^{0.5}\left(\delta_\gamma\right)\right) \\ &= \theta_1' + \alpha \left(0.5 - \mathbb{I}[\gamma^2 < \theta_1']\right) - \alpha \left(0.5 - \mathbb{I}[\gamma < \theta_1'] \right).
\end{align*}
Overall, we have
\begin{align*}
   \theta_1^\text{(bwd)} &= \theta_1 + \alpha \left(0.5 - \mathbb{I}\left[\gamma < \theta_1\right]\right) + \alpha \left(0.5 - \mathbb{I}[\gamma^2 < \theta_1']\right) - \alpha \left(0.5 - \mathbb{I}[\gamma < \theta_1'] \right) \\
   &= 0.5\alpha - \alpha \mathbb{I}[\gamma^2 < 0.5\alpha] + \mathbb{I}[\gamma < 0.5\alpha].
\end{align*}

Now, let $\alpha\in(2\gamma^2,2\gamma)$ such that $0.5\alpha\in(\gamma^2,\gamma)$, we have $\theta_1^{\text{bwd}}=0.5\alpha - \alpha = -0.5\alpha \neq \theta_1^{(\text{fwd})}$.

\subsection{Discussion on memory complexity}
The return traces $G_{t',t}$ and modified eligibility traces $e_{t',t}$ are time-dependent, which is a direct implication from the fact that distributional TD errors are path-dependent. Indeed, to calculate the distributional TD error $\tilde{\Delta}_{t':t}^\pi$, it is necessary to keep track $G_{t',t}$ in the backward-view algorithm. This differs from the classic eligibility traces, which are state-action-dependent \citep{sutton1998,van2020expected}. We remark that the state-action-dependency of eligibility traces result from the fact that value-based TD errors $\Delta_t^\pi$ are path-independent. The time-dependency greatly influences the memory complexity of the algorithm: when an episode is of length $T$, value-based backward-view algorithm requires memory of size $\min(|\mathcal{X}||\mathcal{A}|,T)$ to store all eligibility traces. On the other hand, the distributional backward-view algorithm requires $\mathcal{O}(T)$.

\section{Distributional Retrace with categorical representations}
\label{appendix:categorical}

We start by showing that the distributional Retrace operator is $\beta_{L_p}$-contractive under the $\bar{L}_p$ distance for $p\geq 1$. As a comparison, the one-step distributional Bellman operator $\mathcal{T}^\pi$ is $\gamma^{1/p}$-contractive under $\bar{L}_p$ \citep{rowland2018analysis}.

\begin{restatable}{lemma}{propretracecontractivecramer}\label{prop:retracecontractivecramer} (\textbf{Contraction in $\bar{L}_p$}) $\mathcal{R}^{\pi,\mu} $ is $\beta_{L_p}$-contractive under supremum $L_p$ distance for $p\geq 1$, where $\beta_{L_p}\in[0,\gamma]$. Specifically, we have
$
    \beta_{L_p}=\max_{x\in\mathcal{X},a\in\mathcal{A}} \left( \sum_{t=1}^\infty \mathbb{E}_\mu\left[ c_1...c_{t-1}(1-c_t) \right] \gamma^{t}\right)^{1/p}$.
\end{restatable} 
\begin{proof}
The proof is similar to the proof of Proposition~\ref{prop:retracecontractive}: the result follows by combining the convex combination property of distributional Retrace in Lemma~\ref{lemma:convexcombination} with the $p$-convexity of $L_p$ distance \citep{bdr2022}. 
\end{proof}

\subsection{Categorical representation}

In categorical representations \citep{bellemare2017cramer}, we consider parametric distributions of the form for a fixed $m\geq 1$,
$
    \sum_{i=1}^m p_i\delta_{z_i}$, 
where $(z_i)_{i=1}^m\in\mathbb{R}$ are a fixed set of atoms and $(p_i)_{i=1}^m$ is a categorical distribution such that $\sum_{i=1}^m p_i=1$ and $p_i\geq 0$. Denote the class of such distributions as 
$
    \probspace_\mathcal{C}(\mathbb{R}) \coloneqq \{ \sum_{i=1}^m p_i\delta_{z_i} | \sum_{i=1}^m p_i=1, p_i\geq 0\}$. For simplicity, we assume that the target return is supported on the set of atoms $[R_\text{MIN}/(1-\gamma),R_\text{MAX}/(1-\gamma)]\subset[z_1,z_m]$. 
    
    We introduce the projection that maps from an initial back-up distribution to the categorical parametric class: $\Pi_\mathcal{C}:\probspace_\infty(\mathbb{R})\rightarrow \probspace_\mathcal{C}(\mathbb{R})$ defined as
$
    \Pi_\mathcal{C} \eta \coloneqq \arg\min_{\nu\in\probspace_\mathcal{C}(\mathbb{R})} L_2\left(\nu,\eta\right),\forall \nu\in\probspace_\infty(\mathbb{R})$. 
The projection can be easily calculated as described in \citep{bellemare2017distributional,rowland2018analysis}. For any distribution vector $\eta\in\probspace_\infty(\mathbb{R})^{\mathcal{X}\times\mathcal{A}}$, define $\Pi_\mathcal{C}\eta$ as the component-wise projection. Now, given the composed operator $\Pi_\mathcal{C}\mathcal{R}^{\pi,\mu}:\probspace_\infty(\mathbb{R})^{\mathcal{X}\times\mathcal{A}}\rightarrow\probspace_\mathcal{C}(\mathbb{R})^{\mathcal{X}\times\mathcal{A}}$, we characterize the convergence of the seququence $\eta_k=\left(\Pi_\mathcal{C}\mathcal{R}^{\pi,\mu}\right)^k\eta_0$.

\begin{restatable}{theorem}{theoremcategoricalcontraction}\label{theorem:categoricalcontraction} (\textbf{Convergence of categorical distributions})
The projected distributional Retrace operator $\Pi_\mathcal{C} \mathcal{R}^{\pi,\mu}$ is $\beta_{L_2}$-contractive under $\bar{L}_2$ distance in $\probspace_\mathcal{Q}(\mathbb{R})$. As a result, the above $\eta_k$ converges to a limiting distribution $\eta_\mathcal{R}^\pi$ in $\bar{L}_2$, such that
$
     \bar{L}_2(\eta_k,\eta_\mathcal{R}^\pi)\leq(\beta_{L_2})^k \bar{L}_2(\eta_0,\eta_\mathcal{R}^\pi)
$. Further, the quality of the fixed point is characterized as $\bar{L}_2(\eta_\mathcal{R}^\pi,\eta^\pi)\leq (1-\beta_{L_2})^{-1}\bar{L}_2(\Pi_\mathcal{C}\eta^\pi,\eta^\pi)$.
\end{restatable}

\begin{proof}
The above theorem follows from Lemma~\ref{prop:retracecontractivecramer}. Indeed, since $\Pi_\mathcal{Q}$ is a non-expansion in supremum Cram\'er distance $\bar{L}_2$ \citep{rowland2018analysis}, the composed operator $\Pi_\mathcal{Q}\mathcal{R}^{\pi,\mu}$ is $\beta_{L_2}$-contractive in $\bar{L}_2$. Following the same argument as the proof of Theorem~\ref{theorem:quantilecontraction}, we obtain the remaining desired results.
\end{proof}

The distributional Retrace operator also improves over one-step distributional Bellman operator in two aspects: (1) the bound on the contraction rate $\beta_{L_2}\leq \sqrt{\gamma}$ is smaller, usually leading to faster contraction to the fixed point; (2) the bound on the quality of the fixed point is improved.

\subsection{Cross-entropy update and C51-Retrace}

Unlike in the quantile projection case, where calculating $\Pi_\mathcal{Q}\eta$ requires solving a quantile regression minimization problem, the categorical projection can be calculated in an analytic way \citep{rowland2018analysis,bdr2022}. 
Assume the categorical distribution is parameterized as $\eta_w(x,a) = \sum_{i=1}^m p_i(x,a;w) \delta_{z_i}$.
After computing the back-up target distribution $\Pi_\mathcal{C}\mathcal{R}^{\pi,\mu}\eta(x,a)$ for a given distribution vector $\eta$, the algorithm carries out a gradient-based incremental update
\begin{align*}
    w \leftarrow w - \alpha \nabla_w \mathbb{CE}\left[\Pi_\mathcal{C}\mathcal{R}^{\pi,\mu}\eta(x,a) | \eta_w(x,a)\right],
\end{align*}
where $\mathbb{CE}(p|q)\coloneqq -\sum_i p_i\log q_i$ denotes the cross-entropy between distribution $p$ and $q$. For simplicity, we adopt a short-hand notation $\mathbb{CE}(\eta|\eta_w)=\mathbb{CE}_w(\eta)$. Note also that in practice, $\eta$ can be a slowly updated copy of $\eta_w$ \citep{mnih2015humanlevel}. As such, the gradient-based update can be understood as approximating the iteration $\eta_{k+1}=\mathcal{R}^{\pi,\mu}\eta_k$. We propose the following unbiased estimate to the cross-entropy $\hat{\mathbb{CE}}_w\left[\Pi_\mathcal{C}\mathcal{R}^{\pi,\mu}\eta(x,a)\right]$, calculated as follows
\begin{align*}
  \mathbb{CE}_w\left(\eta(x,a)\right) + \sum_{t=0}^\infty c_{1:t} \left( \mathbb{CE}_w\left(\left(\textrm{b}_{t+1}\right)_{\#}\eta\left(X_{t+1},A_{t+1}^\pi\right)\right) - \mathbb{CE}_w\left(\left(\textrm{b}_t\right)_{\#}\eta(X_t,A_t)\right)\right).
\end{align*}

\begin{restatable}{lemma}{lemmaunbiasedcat}\label{lemma:unbiasedcat} (\textbf{Unbiased stochastic estimate for categorical update}) Assume that the trajectory terminates within $H<\infty$ steps almost surely, then we have $\mathbb{E}_\mu\left[\hat{\mathbb{CE}}_w\left(\Pi_\mathcal{C}\mathcal{R}^{\pi,\mu}\eta(x,a)\right)\right]=\mathbb{CE}_w\left(\Pi_\mathcal{C}\mathcal{R}^{\pi,\mu}\eta(x,a)\right)$. Without loss of generality, assume $w$ is a scalar parameter. If there exists a constant $M>0$ such that $\left|\nabla_w \mathbb{CE}_w\left(\eta\right)\right|\leq M,\forall \eta\in\probspace_\infty(\mathbb{R})$, then the gradient estimate is also unbiased $ \mathbb{E}_\mu\left[\nabla_w \hat{\mathbb{CE}}_w\left(\Pi_\mathcal{C}\mathcal{R}^{\pi,\mu}\eta(x,a)\right)\right]=\nabla_w \mathbb{CE}_w\left(\Pi_\mathcal{C}\mathcal{R}^{\pi,\mu}\eta(x,a)\right)$.
\end{restatable}
\begin{proof}
The cross-entropy is defined for any distribution $\mathbb{CE}_w(\eta)$. For any signed measure $\nu=\sum_{i=1}^m w_i\eta_i$ with $\eta_i\in\probspace_\infty(\mathbb{R})$, we define the generalized cross-entropy as 
\begin{align*}
    \mathbb{CE}_w\left(\nu\right) \coloneqq \sum_{i=1}^m  w_i \mathbb{CE}_w\left(\eta_i\right), 
\end{align*}
Next, we note the cross-entropy is linear in the input distribution (or signed measure). In particular, for a set of $N$ (potentially infinite) coefficients and distributions (signed measures) $(a_i,\eta_i)$,
\begin{align*}
    \mathbb{CE}_w\left(\sum_{i=1}^N a_i \eta_i \right) \coloneqq \sum_{i=1}^m a_i \mathbb{CE}_w\left( \eta_i \right).
\end{align*}
When $a_i$ denotes a distribution, the above rewrites as $\mathbb{CE}_w\left(\mathbb{E}[\eta_i]\right)=\mathbb{E}[\mathbb{CE}(\eta_i)]$. Finally, combining everything together, we have $\mathbb{E}_\mu\left[\hat{\mathbb{CE}}_w\left(\Pi_\mathcal{C}\mathcal{R}^{\pi,\mu}\eta(x,a)\right)\right] $ evaluate to
\begin{align*} 
    &= \mathbb{E}_\mu\left[\mathbb{CE}_w\left(\eta(x,a)\right) + \sum_{t=0}^\infty c_{1:t} \left( \mathbb{CE}_w\left(\left(\textrm{b}_{t+1}\right)_{\#}\eta\left(X_{t+1},A_{t+1}^\pi\right)\right) - \mathbb{CE}_w\left(\left(\textrm{b}_t\right)_{\#}\eta(X_t,A_t)\right)\right)\right] \\
    &=_{(a)} \mathbb{E}_\mu\left[\mathbb{CE}\left(\hat{\mathcal{R}}^{\pi,\mu}\eta(x,a)\right)\right] 
    =_{(b)} \mathbb{E}_\mu\left[\mathbb{CE}\left(\mathcal{R}^{\pi,\mu}\eta(x,a)\right)\right]. 
\end{align*}
In the above, (a) follows from the definition of the cross-entropy with signed measure $\hat{\mathcal{R}}^{\pi,\mu}\eta(x,a)$ and (b) follows from the linearity property of cross-entropy. 

Next, to show that the gradient estimate is unbiased too, the high level idea is to apply dominated convergence theorem (DCT) to justify the exhchange of gradient and expectation \citep{rudin1976principles}. This is similar to the quantile representation case (see proof for Lemma~\ref{lemma:unbiasedqr}). To this end, consider the absolute value of the gradient estimate $
    \left|\nabla_w \hat{\mathbb{CE}}_w\left(\mathcal{R}^{\pi,\mu}\eta(x,a)\right)\right|
$, which serves as an upper bound to the gradient estimate. In order to apply DCT, we need to show the expectation of the absolute gradient is finite. Note we have
\begin{align*}
    &\mathbb{E}_\mu\left[
    \left|\nabla_w \hat{\mathbb{CE}}_w\left(\mathcal{R}^{\pi,\mu}\eta(x,a)\right)\right|
\right] \\ &= \mathbb{E}_\mu\left[\left|\nabla_w \mathbb{CE}_w\left(\eta(x,a)\right) + \sum_{t=0}^H c_{1:t} \left( \nabla_w \mathbb{CE}_w\left(\left(\textrm{b}_{t+1}\right)_{\#}\eta\left(X_{t+1},A_{t+1}^\pi\right)\right) - \nabla_w \mathbb{CE}_w\left(\left(\textrm{b}_t\right)_{\#}\eta(X_t,A_t)\right)\right)\right|\right] \\
&\leq_{(a)} \mathbb{E}_\mu\left[\left|\nabla_w \mathbb{CE}_w \left(\eta(x,a)\right)\right| + \sum_{t=0}^H c_{1:t} \left| \nabla_w \mathbb{CE}_w\left(\left(\textrm{b}_{t+1}\right)_{\#}\eta\left(X_{t+1},A_{t+1}^\pi\right)\right) - \nabla_w \mathbb{CE}_w\left(\left(\textrm{b}_t\right)_{\#}\eta(X_t,A_t)\right)\right|\right] \\
&\leq_{(b)} \mathbb{E}_\mu\left[M + \sum_{t=0}^H \rho^t \cdot M\right] <\infty,
\end{align*}
where (a) follows from the application of triangle inequality; (b) follows from the fact that the QR loss gradient against a fixed distribution is bounded $\nabla_w \mathbb{CE}_w\left(\nu\right)\in[-M,M],\forall \nu\in\probspace_\infty(\mathbb{R})$ \citep{dabney2018distributional}. 

Hence, with the application DCT, we can exchange the gradient and expectation operator, which yields $ \mathbb{E}_\mu\left[\nabla_w \hat{\mathbb{CE}}_w^\tau\left(\mathcal{R}^{\pi,\mu}\eta(x,a)\right)\right] = \nabla_w \mathbb{E}_\mu\left[ \hat{\mathbb{CE}}_w^\tau\left(\mathcal{R}^{\pi,\mu}\eta(x,a)\right)\right] = \nabla_w \mathbb{CE}_w \left(\mathcal{R}^{\pi,\mu}\eta(x,a)\right)$.

\end{proof}

We remark that the condition on the bounded gradient $|\nabla_w \mathbb{CE}_w\left(\eta\right)|\leq M$ is not restrictive. When $\eta_w$ is adopts a softmax parameterization and $w$ represents the logits, $M=1$.

Finally, the deep RL agent C51 parameterizes the categorical distribution $p_i(x,a;w)$ with a neural network $w$ at each state action pair $(x,a)$ \citep{bellemare2017cramer}. When combined with the above algorithm, this produces  C51-Retrace. 

\section{Additional experiment details}
\label{appendix:experiments}

In this section, we provide detailed information about experiment setups and additional results. All experiments are carried out in Python, using NumPy for numerical computations \citep{harris2020array} and Matplotlib for visualization \citep{hunter2007matplotlib}. All deep RL experiments are carried out with Jax \citep{bradbury2018jax}, specifically
making use of the DeepMind Jax ecosystem \citep{babuschkin2010deepmind}.

\subsection{Tabular}

We provide additional details on the tabular RL experiments.

\paragraph{Setup.} We consider a tabular MDP with $|\mathcal{X}|=3$ states and $|\mathcal{A}|=2$ actions. The reward $r(x,a)$ is deterministic and generated from a standard Gaussian distribution. The transition probability $P(\cdot|x,a)$ is sampled from a Dirichlet distribution with parameter $(\Gamma,\Gamma...\Gamma)$ for $\Gamma=0.5$. The discount factor is fixed as $\gamma=0.9$. The MDP has a starting state-action pair $(x_0,a_0)$. The behavior policy $\mu$ is a uniform policy. The target policy is generated as follows: we first sample a deterministic policy $\pi_d$ and then compute $\pi=(1-\epsilon)\pi_d+\epsilon \mu$, with parameter $\epsilon$ to control the level of off-policyness. 

\paragraph{Quantile distribution and projection.} We use $m=100$ atoms throughout the experiments. Assuming access to the MDP parameters (e.g., reward and transition probability), we can analytically compute the projection $\Pi_\mathcal{Q}$ using a sorting algorithm. See \citep{dabney2018distributional,bdr2022} for details.

\paragraph{Evaluation metrics.} Let $\eta_k=(\mathcal{R}^{\pi,\mu})^k\eta_0$ be the $k$-th iterate.
We use a few different metrics in Figure~\ref{fig:tabular}. Given any particular distributional Retrace operator $\mathcal{R}^{\pi,\mu}$, there exists a fixed point to the composed operator $\Pi_\mathcal{Q}\mathcal{R}^{\pi,\mu}$. Recall that we denote this distribution as $\eta_\mathcal{R}^\pi$. Fig~\ref{fig:tabular}(a)-(b) calculates the iterates' distance from the fixed point, evaluated at $(x_0,a_0)$.
\begin{align*}
    L_p\left(\eta_k(x_0,a_0), \eta_\mathcal{R}^\pi(x_0,a_0)\right).
\end{align*}
Fig~\ref{fig:tabular}(c) calculates the distance from the projected target distribution $\Pi_\mathcal{Q}\eta^\pi$. Recall that $\Pi_\mathcal{Q}\eta^\pi$ is in some sense the best possible approximation that the current quantile representation can obtain.
\begin{align*}
    L_p\left(\eta_k(x_0,a_0), \Pi_\mathcal{Q}\eta^\pi(x_0,a_0)\right).
\end{align*}

\subsection{Deep reinforcement learning}

We provide additional details on the deep RL experiments.

\paragraph{Evaluation metrics.} For the $i$-th of the $57$ Atari games, we obtain the performance of the agent $G_i$ at any given point in training. The normalized performance is computed as 
$Z_i = (G_i - U_i) / (H_i - U_i)$ where $H_i$ is the human performance and $U_i$ is the performance of a random policy. Then the mean/median metric is calculated as the mean or median statistics over $(Z_i)_{i=1}^{57}$.

The super human ratio is computed as the number of games such as $Z_i\geq 1$, i.e., $G_i\geq H_i$ where the agent obtains super human performance on the game. Formally, it is compute as $\frac{1}{57}\sum_{i=1}^{57} \mathbb{I}[Z_i\geq 1]$.

\paragraph{Shared properties of all baseline agents.} All baseline agents use the same torso architecture as DQN \citep{mnih2015humanlevel} and differ in the head outputs, which we specify below. All agents an Adam optimizer \citep{kingma2014adam} with a fixed learning rate; the optimization is carried out on mini-batches of size $32$ uniformly sampled from the replay buffer. For exploration, the agent acts $\epsilon$-greedy with respect to induced Q-functions, the details of which we specify below. The exploration policy adopts $\epsilon$ that starts with $\epsilon_\text{max}=1$ and linearly decays to $\epsilon_\text{min}=0.01$ over training. At evaluation time, the agent adopts $\epsilon=0.001$; the small exploration probability is to prevent the agent from getting stuck.

\paragraph{Details of baseline C51 agent.}  The agent head outputs a matrix of size $|\mathcal{A}| \times m$, which represents the logits to  $\left(p_i(x,a;\theta)\right)_{i=1}^m$. The support $(z_i)_{i=1}^m$ is generated as a uniform array over $[-V_\text{MAX},V_\text{MAX}]$. Though $V_\text{MAX}$ should in theory be determined by $R_\text{MAX}$; in practice, it has been found that setting $V_\text{MAX}=R_\text{MAX}/(1-\gamma)$ leads to highly sub-optimal performance. This is potentially because usually the random returns are far from the extreme values $R_\text{MAX}/(1-\gamma)$, and it is better to set $V_\text{MAX}$ at a smaller value. Here, we set $V_\text{MAX}=10$ and $m=51$.  For details of other hyperparameters, see \citep{bellemare2017distributional}.  The induced Q-function is computed as $Q_\theta(x,a)=\sum_{i=1}^m p_i(x,a;\theta)z_i$. 

\paragraph{Details of baseline QR-DQN agent.} The agent head outputs a matrix of size $|\mathcal{A}| \times m$, which represents the quantile locations  $\left(z_i(x,a;\theta)\right)_{i=1}^m$. Here, we set $m=201$. For details of other hyperparameters, see \citep{dabney2018distributional}. The induced Q-function is computed as $Q_\theta(x,a)=\frac{1}{m}\sum_{i=1}^m z_i(x,a;\theta)$.

\paragraph{Details of multi-step agents.} Multi-step variants use exactly the same hyperparameters as the one-step baseline agent. The only difference is that the agent uses multi-step back-up targets.

The agent stores partial trajectories $(X_t,A_t,R_t,x_t)_{t=0}^{n-1}\sim \mu$ generated under the behavior policy. Here, the behavior policy $\mu$ is the $\epsilon$-greedy policy with respect to a potentially old Q-function (this is because the data at training time is sampled from the replay); the target policy $\pi$ is the greedy policy with respect to the current Q-function.

\begin{figure}
    \centering
    \subfigure[C51 ]{\includegraphics[keepaspectratio,width=.96\textwidth]{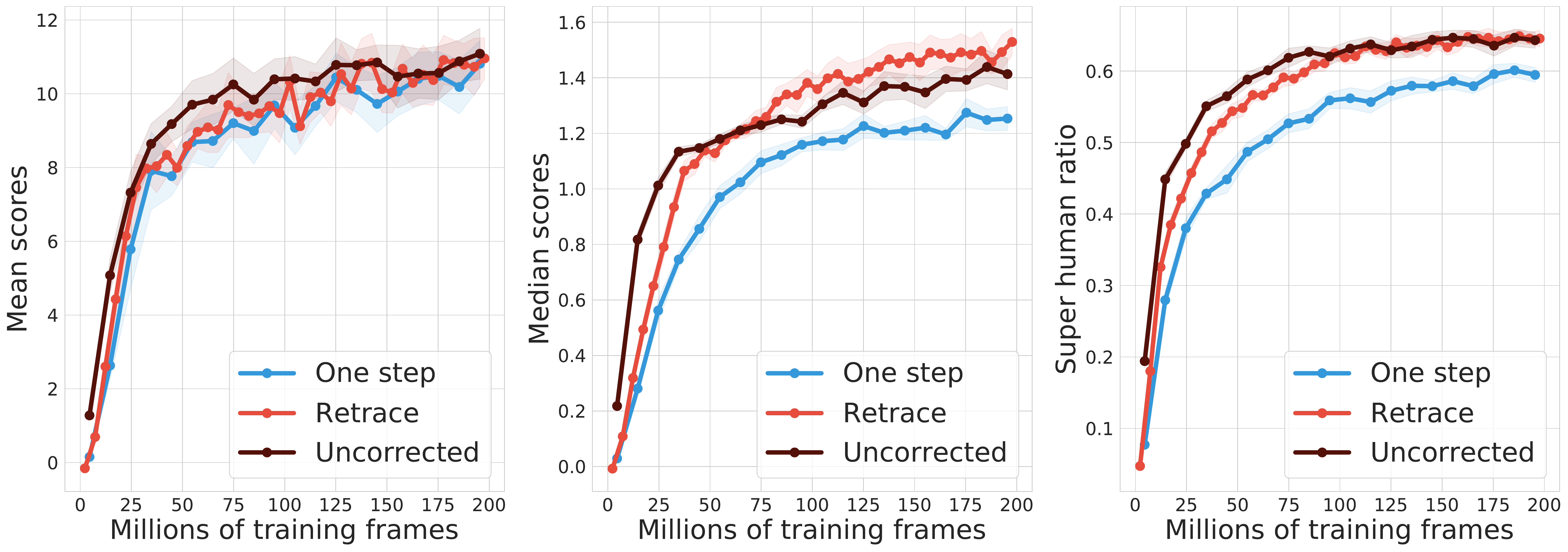}}
    \subfigure[QR-DQN ]{\includegraphics[keepaspectratio,width=.96\textwidth]{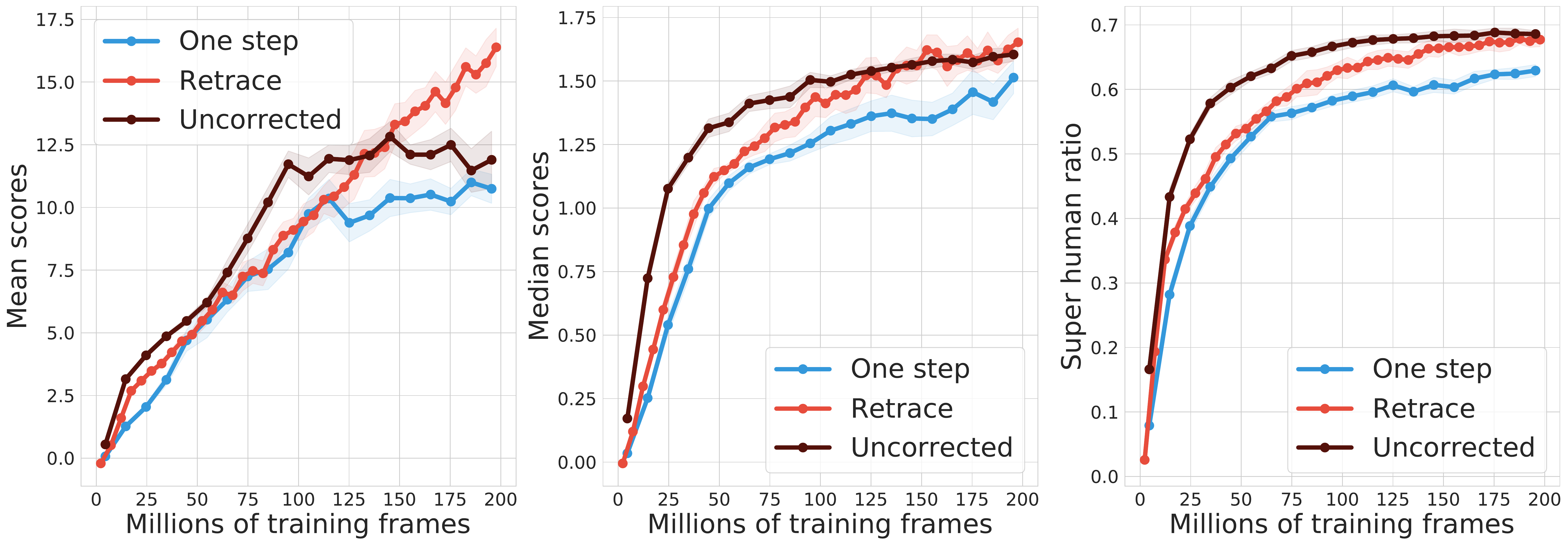}}
    \caption{Deep RL experiments on Atari-57 games for (a) C51 and (b) QR-DQN. We compare the one-step baseline agent against the multi-step variants (Retrace and uncorrected $n$-step). For all multi-step variants, we use $n=3$. For each agent, we calculate the mean, median and super human ratio performance across all games, and we plot the $\text{mean}\pm\text{standard\ error}$ across $3$ seeds. In almost all settings, Multi-step variants provide clear advantage over the one-step baseline algorithm.}
    \label{fig:deeprl-exp-all}
\end{figure}

\begin{figure}
    \centering
    \subfigure[C51 ]{\includegraphics[keepaspectratio,width=.48\textwidth]{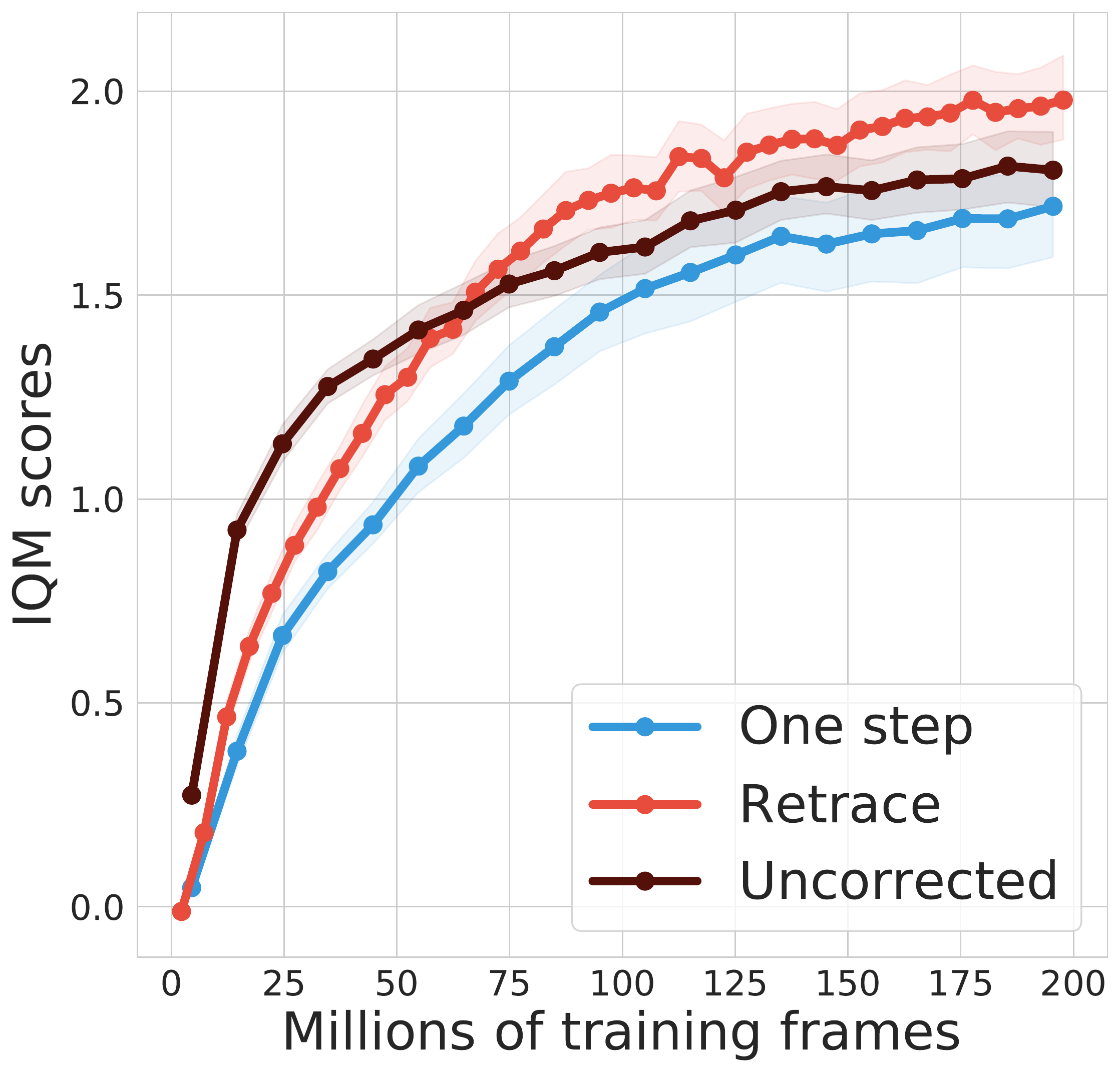}}
    \subfigure[QR-DQN ]{\includegraphics[keepaspectratio,width=.48\textwidth]{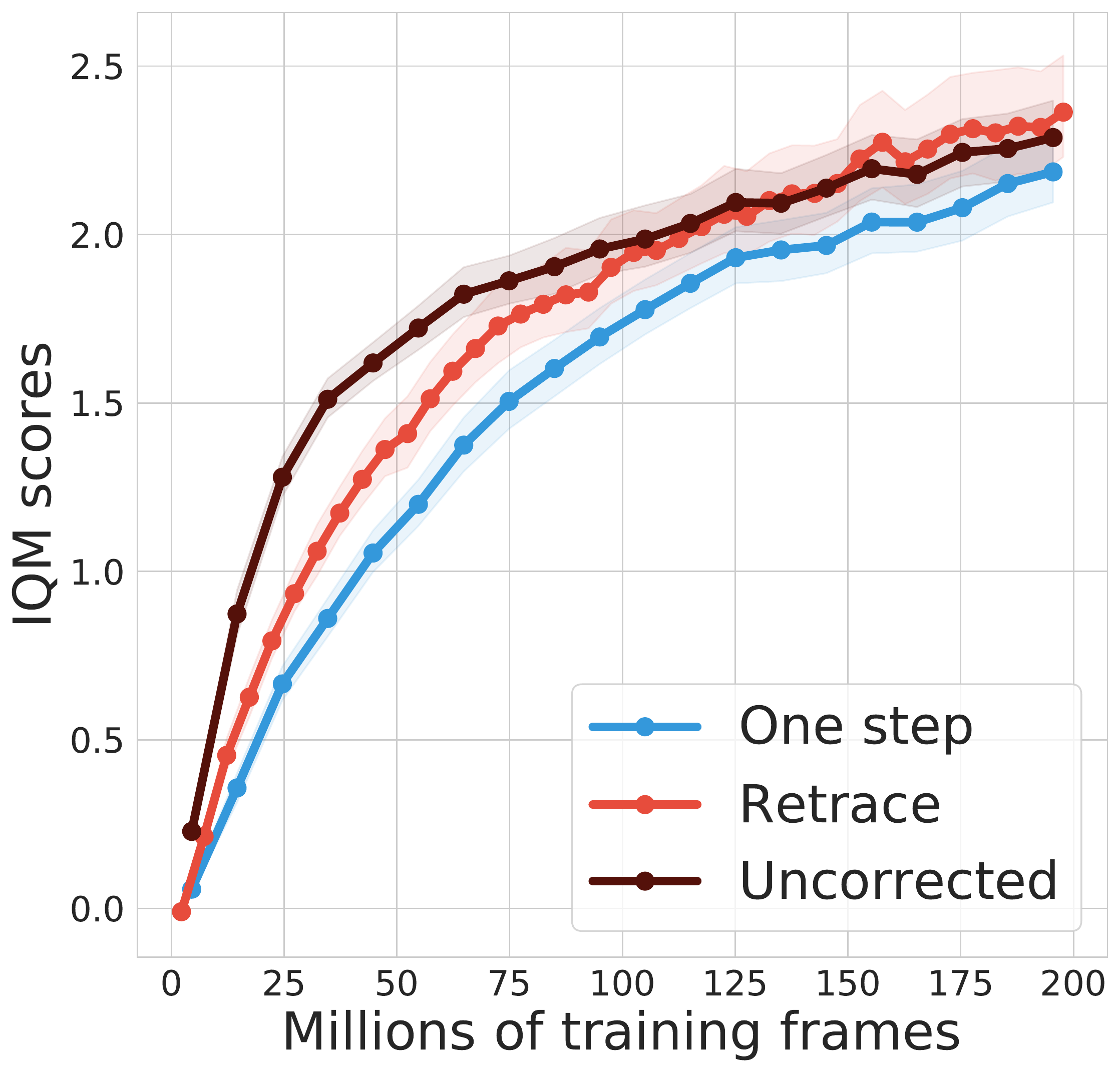}}
    \caption{Deep RL experiments on Atari-57 games for (a) C51 and (b) QR-DQN, with the same setup as in Figure~\ref{fig:deeprl-exp-all}. Here, we compute the interquartile mean (IQM) with 95\% bootstrapped confidence interval \citep{agarwal2021deep}. In a nutshell, IQM calculates the mean scores after removing extreme score values, making the performance statistics more robust. Even after excluding extreme scores, Retrace obtains favorable performance compared to the uncorrected and one-step algorithm.  }
    \label{fig:deeprl-exp-iqm}
\end{figure}

\section{Proof}
\label{appendix:proof}

To simplify the proof, we assume that the immediate random reward takes a finite number of values. It is straightforward to generalize results to the case where the reward takes an infinite number of values (e.g., the random reward has a continuous distribution).

\begin{restatable}{assumption}{assumereward}\label{assume:reward} \textbf{(Reward takes a finite number of values)} For all state-action pair $(x,a)$, we assume the random reward $R(x,a)$ takes a finite number of values. Let $\tilde{R}$ be the finite set of values that the reward $\{R(x,a),(x,a)\in\mathcal{X}\times\mathcal{A}\}$ can take.
\end{restatable}

For any integer $t\geq 1$, Let $\tilde{R}^t$ denotes the Cartesian product of $t$ copies of $\tilde{R}$: 
\begin{align*}
    \tilde{R}^t\coloneqq \underbrace{\tilde{R}\times\tilde{R}\times...\times\tilde{R}}_{t\ \text{copies\ of\ }\tilde{R}}.
\end{align*}
For any fixed $t$, we let $r_{0:t-1}$ denote the sequence of realizable rewards from time $0$ to time $t-1$. Since $\tilde{R}$ is a finite set, $\tilde{R}^t$ is also a finite set. 

\lemmaconvexcombination*
\begin{proof}

In general $c_t=c(F_t,A_t)$ where $F_t$ is a filtration of $(X_s,A_s)_{s=0}^t$. To start with, we assume $c_t=c(X_t,A_t)$ to be a Markovian trace coefficient \citep{munos2016safe}. We start with the simpler case because the proof is greatly simplified with notations and can extend to the general case with some care. We discuss the extension to the general case where $c_t=c(F_t,A_t)$ towards the end of the proof.

For all $t\geq 1$, we define the coefficient
\begin{align*}
    w_{y,b,r_{0:t-1}} \coloneqq \mathbb{E}_\mu\left[c_1...c_{t-1}\left(\pi(b|X_t) - c(X_t,b)\mu(b|X_t)\right)\cdot \mathbb{I}[X_t=y]\Pi_{s=0}^{t-1}\mathbb{I}[R_s=r_s]\right].
\end{align*}

Through careful algebra, we can rewrite the Retrace operator as follows
\begin{align*}
    \mathcal{R}^{\pi,\mu}\eta(x,a) = \sum_{t=1}^{\infty} \sum_{y\in\mathcal{X}} \sum_{b\in\mathcal{A}} \sum_{r_{0:t-1}\in\tilde{R}^t} w_{y,b,r_{0:t-1}} \left(\textrm{b}_{G_{0:t-1},\gamma^t}\right)_{\#} \eta(y,b).
\end{align*}
Note that each term of the form $\left(\textrm{b}_{G_{0:t-1},\gamma^t}\right)_{\#} \eta(y,b)$ corresponds to applying a pushforward operation $\left(\textrm{b}_{G_{0:t-1},\gamma^t}\right)_{\#}$ on the distribution $\eta(x,a)$, which means $\left(\textrm{b}_{G_{0:t-1},\gamma^t}\right)_{\#} \eta(y,b)\in\probspace_\infty(\mathbb{R})$. Now that we have expressed $\mathcal{R}^{\pi,\mu}\eta(x,a)$ as a linear combination of distributions, we proceed to show that the combination is in fact convex.

Under the assumption $c_t\in [0, \rho_t]$, we have $\pi(b|y)-c(y,b)\mu(b|y)\geq 0$ for all $(y,b)\in\mathcal{X}\times\mathcal{A}$. Therefore, all weights are non-negative. Next, we examine the sum of all coefficients $\sum w_{y,b,r_{0:t-1}} =  \sum_{t=1}^{\infty} \sum_{x\in\mathcal{X}}\sum_{b\in\mathcal{A}}\sum_{r_{0:t-1}\in\tilde{R}^t} w_{y,b,r_{0:t-1}}$. 
\begin{align*}
  \sum w_{y,b,r_{0:t-1}}
   &=_{(a)} \sum_{t=1}^\infty \sum_{y\in\mathcal{X}} \sum_{b\in\mathcal{A}} \mathbb{E}_\mu\left[c_1...c_{t-1} \left(\pi(b|X_t)-c(X_t,b)\mu(b|X_t)\right) \cdot \mathbb{I}[X_t=y]\right] \\ 
   &=_{(b)} \sum_{t=1}^\infty \mathbb{E}_\mu\left[c_1...c_{t-1}(1-c_t)\right] 
   =_{(c)} 1.
\end{align*}
In the above, (a) follows from the fact that $\sum_{r_s\in\tilde{R}}\mathbb{E}[\mathbb{I}[R_s=r_s]]=1$; (b) follows from the fact that for all time steps $t\geq 1$, the following is true,
\begin{align*}
    &\sum_{y\in\mathcal{X}} \sum_{b\in\mathcal{A}} \mathbb{E}_\mu\left[c_1...c_{t-1} \left(\pi(b|X_t)-c(X_t,b)\mu(b|X_t)\right) \cdot \mathbb{I}[X_t=y]\right]  \\
    &= \sum_{b\in\mathcal{A}} \mathbb{E}_\mu\left[c_1...c_{t-1} \left(\pi(b|X_t)-c(X_t,b)\mu(b|X_t)\right) \right] \\
    &=  \mathbb{E}_\mu\left[c_1...c_{t-1} \left(1-\sum_{b\in\mathcal{A}}c(X_t,b)\mu(b|X_t)\right) \right] \\
    &= \mathbb{E}_\mu\left[c_1...c_{t-1}(1-c_t)\right].
\end{align*}
Finally, (c) is based on the observation that the summation telescopes. Now, by taking the index set to be the set of indices that parameterize $w_{y,b,r_{0:t-1}}$, 
\begin{align*}
    I(x,a) = \cup_{t=1}^\infty\left(y,b,r_{0:t-1}\right)_{y\in\mathcal{X},b\in\mathcal{A},r_{0:t-1}\in\tilde{R}^t}.
\end{align*}
We can write $\mathcal{R}^{\pi,\mu}\eta(x,a)=\sum_{i\in I(x,a)}w_i \eta_i$. Note  further that for any $i \in I(x,a)$, $\eta_i=(\textrm{b}_{G_{0:t-1},\gamma^t})_\# \eta(y,b)$ is a fixed distribution. The above result suggests that $\mathcal{R}^{\pi,\mu}\eta(x,a)$ is a convex combination of fixed distributions.
\paragraph{Extension to the general case.} When $c_t=c(F_t,A_t)$ is filtration dependent, we let $\mathcal{F}_t$ to be the space of the filtration value up to time $t$. For simplicity with the notation, we assume $\mathcal{F}_t$ contains a finite number of elements, such that below we can adopt the summation notation instead of integral. Define the combination coefficient
\begin{align*}
    w_{y,b,f_t,r_{0:t-1}} \coloneqq \mathbb{E}_\mu\left[c_1...c_{t-1}\left(\pi(b|X_t) - c(F_t,b)\mu(b|X_t)\right)\cdot \mathbb{I}[X_t=y]\Pi_{s=0}^{t-1}\mathbb{I}[R_s=r_s]\right].
\end{align*}
It is straightforward to verify the following
\begin{align*}
    \mathcal{R}^{\pi,\mu}\eta(x,a) = \sum_{t=1}^{\infty} \sum_{y\in\mathcal{X}} \sum_{b\in\mathcal{A}}\sum_{f_t\in \mathcal{F}_t} \sum_{r_{0:t-1}\in\tilde{R}^t} w_{y,b,f_t,r_{0:t-1}} \left(\textrm{b}_{G_{0:t-1},\gamma^t}\right)_{\#} \eta(y,b).
\end{align*}
In addition, the combination coefficients $w_{y,b,f_t,r_{0:t-1}}$ sum to $1$ and are all non-negative.
\end{proof}

\propretracecontractive*
\begin{proof}
From the proof of Lemma~\ref{lemma:convexcombination}, we have
\begin{align*}
    \mathcal{R}^{\pi,\mu}\eta(x,a) = \sum_{t=1}^{\infty} \sum_{y\in\mathcal{X}} \sum_{b\in\mathcal{A}} \sum_{r_{0:t-1}\in\tilde{R}^t} w_{y,b,r_{0:t-1}} \left(\textrm{b}_{G_{0:t-1},\gamma^t}\right)_{\#} \eta(y,b).
\end{align*}
Now, we have for any $\eta_1,\eta_2\in\probspace_\infty(\mathbb{R})^{\mathcal{X}\times\mathcal{A}}$, for any fixed $(x,a)$, we have $W_p\left(\mathcal{R}^{\pi,\mu}\eta_1(x,a),\mathcal{R}^{\pi,\mu}\eta_2(x,a)\right)$ upper bounded as follows
\begin{align*}
     &\leq_{(a)} \sum_{t=1}^\infty \sum_{y\in\mathcal{X}}\sum_{b\in\mathcal{A}} w_{y,b,r_{0:t-1}} W_p\left(\left(\textrm{b}_{\sum_{s=0}^{t-1}\gamma^s r_s,\gamma^t}\right)_\#\eta_1(y,b),\left(\textrm{b}_{\sum_{s=0}^{t-1}\gamma^s r_s,\gamma^t}\right)_\#\eta_2(y,b)\right) \\ &\leq_{(b)} \sum_{t=1}^\infty \sum_{y\in\mathcal{X}}\sum_{b\in\mathcal{A}} w_{y,b,r_{0:t-1}} \gamma^t W_p\left(\eta_1(y,b),\eta_2(y,b)\right)
     \\ &\leq_{(c)} \sum_{t=1}^\infty \sum_{y\in\mathcal{X}}\sum_{b\in\mathcal{A}} w_{y,b,r_{0:t-1}} \gamma^t \bar{W}_p\left(\eta_1,\eta_2\right) 
\end{align*}
In the above, (a) follows by applying the convexity of the $p$-Wasserstein distance \citep{bdr2022}; (b) follows by the contraction property of the pushforward operation and $W_p$ \citep{bdr2022}; (c) follows from the definition of $\bar{W}_p$. By taking the maixmum over $(x,a)$ on both sides of the inequality, we obtain
\begin{align*}
    \bar{W}_p(\mathcal{R}^{\pi,\mu}\eta_1,\mathcal{R}^{\pi,\mu}\eta_2)\leq \beta \bar{W}_p(\eta_1,\eta_2).
\end{align*}
This concludes the proof.
\end{proof}

\begin{restatable}{lemma}{lemmapushforwardtranslation}\label{lemma:pushforwardtranslation} For any fixed $(x,a)$ and scalar $c\in\mathbb{R}$,
\begin{align}
    \left(\textrm{b}_{c,1}\right)_{\#}  \eta^\pi(x,a) = \mathbb{E}_\pi\left[ \left(\textrm{b}_{c+R_0,\gamma}\right)_{\#} \eta^\pi(X_1,A_1)\; \middle| \; X_0=x,A_0=a\right]. \label{eq:shifted-dist-bellman}
\end{align}
\end{restatable}
\begin{proof}
Let $B_y\coloneqq\{x<y|x\in\mathbb{R}\}$ be a subset of $\mathbb{R}$ indexed by $y\in\mathbb{R}$. Since the set of all such sets $\{B_y,y\in\mathbb{R}\}$ is dense in the sigma-field of $\mathbb{R}$ \citep{rudin1976principles}, if we can show for two measures $\eta_1,\eta_2$ 
\begin{align*}
    \eta_1\left(B_y\right)=\eta_2\left(B_y\right),\forall y
\end{align*}
then, $\eta_1(B)=\eta_2(B)$ for all Borel sets in $\mathbb{R}$. Hence, in the following, we seek to show
\begin{align}
   \left( \left(\textrm{b}_{c,1}\right)_{\#}  \eta^\pi(x,a)\right)B_y = \left(\mathbb{E}_\pi\left[ \left(\textrm{b}_{c+R_0,\gamma}\right)_{\#} \eta^\pi(X_1,A_1)\right]\right)B_y,\forall y\in\mathbb{R}\label{eq:proof2}
\end{align}
Let $F^\pi(y;x,a)\coloneqq P^\pi(G^\pi(x,a)\leq y)=\eta^\pi(x,a)(B_y),y\in\mathbb{R}$ be the CDF  
of random variable $G^\pi(x,a)$. The distributional Bellman equation in Equation~\eqref{eq:dist-bellman} implies
\begin{align*}
    F^\pi(y;x,a) = \mathbb{E}_\pi\left[F^\pi\left(\frac{y-R_0}{\gamma};X_1,A_1\right)\right],\forall y\in\mathbb{R}.
\end{align*}
For any constant $c\in\mathbb{R}$, let $y=y'-c$ and plug into the above equality,
\begin{align*}
    F^\pi(y'
    -c;x,a) = \mathbb{E}_\pi\left[F^\pi\left(\frac{y'-c-R_0}{\gamma};X_1,A_1\right)\right],\forall y'\in\mathbb{R}.
\end{align*}
Note the LHS is $\left((\textrm{b}_{c,1})_{\#}\eta^\pi(x,a)\right)(B_y)$ while the RHS is $\left(\mathbb{E}_\pi\left[\left(\textrm{b}_{c+R_0,\gamma}\right)_{\#}\eta^\pi(X_1,A_1)\right]\right)(B_y)$. This implies that Equation~\eqref{eq:proof2} holds and we conclude the proof.
\end{proof}

\propretracefixedpoint*

\begin{proof}
     To verify that $\eta^\pi$ is a fixed point, it is equivalent to show
\begin{align*}
     \mathbb{E}_{\mu}\left[ \sum_{t=0}^n c_{1:t}\left(\left(\textrm{b}_{G_{0:t},\gamma^{t+1}}\right)_{\#}\eta^\pi(X_{t+1},A_{t+1}^\pi) - \left(\textrm{b}_{G_{0:t-1},\gamma^{t}}\right)_{\#}\eta^\pi(X_t,A_t)\right)\right] = \mathbf{0}.
\end{align*}
Here, the RHS term $\mathbf{0}$ denotes the zero measure, a measure such that for all Borel sets $B\subset\mathbb{R}$, $ \mathbf{0}(B)=0$. We now verify that each of the summation term is a zero measure, i.e., 
\begin{align*}
    \mathbb{E}_{\mu}\left[  c_{1:t}\left(\left(\textrm{b}_{G_{0:t},\gamma^{t+1}}\right)_{\#}\eta^\pi(X_{t+1},A_{t+1}^\pi) - \left(\textrm{b}_{G_{0:t-1},\gamma^{t}}\right)_{\#}\eta^\pi(X_t,A_t)\right)\right] = \mathbf{0}.
\end{align*}
To see this, we follow the derivation below,
\begin{align}
    &\mathbb{E}_{\mu}\left[  c_{1:t}\left(\left(\textrm{b}_{G_{0:t},\gamma^{t+1}}\right)_{\#}\eta^\pi(X_{t+1},A_{t+1}^\pi) - \left(\textrm{b}_{G_{0:t-1},\gamma^{t}}\right)_{\#}\eta^\pi(X_t,A_t)\right)\right] \nonumber\\ 
    &=_{(a)}
   \mathbb{E}\left[\mathbb{E}\left[  c_{1:t} \left[\left(\left(\textrm{b}_{G_{0:t},\gamma^{t+1}}\right)_{\#}\eta^\pi(X_{t+1},A_{t+1}^\pi)\right] - \left(\textrm{b}_{G_{0:t-1},\gamma^{t}}\right)_{\#}\eta^\pi(X_t,A_t)\right) \; \middle| \;  \left(X_s,A_s,R_{s-1}\right)_{s=1}^t \right]\right] \nonumber\\
   &=_{(b)}  \mathbb{E}\left[c_{1:t} \mathbb{E}\left[  \left(\textrm{b}_{G_{0:t},\gamma^{t+1}}\right)_{\#}\eta^\pi(X_{t+1},A_{t+1}^\pi) \; \middle| \;  \left(X_s,A_s,R_{s-1}\right)_{s=1}^t \right]- c_{1:t} \left(\textrm{b}_{G_{0:t-1},\gamma^{t}}\right)_{\#}\eta^\pi(X_t,A_t) \right]\nonumber\\
   &=_{(c)} \mathbb{E}\left[c_{1:t} \underbrace{\mathbb{E}\left[  \left(\textrm{b}_{G_{0:t-1}+\gamma^t R_t,\gamma^{t+1}}\right)_{\#}\eta^\pi(X_{t+1},A_{t+1}^\pi) \; \middle| \;  \left(X_s,A_s,R_{s-1}\right)_{s=1}^t \right]}_{\text{first term}}- c_{1:t} \left(\textrm{b}_{G_{0:t-1},\gamma^{t}}\right)_{\#}\eta^\pi(X_t,A_t) \right].\label{eq:proof1}
\end{align}
In the above, in (a) we condition on $(X_s,A_s,R_s)_{s=1}^t$ and the equality follows from the tower property of expectations; in (b), we use the fact that the trace product $c_{1:t}$ and $\left(\textrm{b}_{G_{0:t-1},\gamma^t}\right)_{\#}\eta^\pi(X_t,A_t)$ are deterministic function of the conditioning variable $(X_s,A_s,R_s)_{s=1}^t$; in (c), we split the summation $G_{0:t}=G_{0:t-1} + \gamma R_t$. Now we examine the first term in Equation~\eqref{eq:proof1}, by applying Lemma~\ref{lemma:pushforwardtranslation}, we have
\begin{align*}
    \text{first\ term} = \left(\textrm{b}_{G_{0:t-1},\gamma^t}\right)_{\#}\eta^\pi(X_t,A_t).
\end{align*}

This implies Equation~\eqref{eq:proof1} evaluates to a zero measure. Hence $\eta^\pi$ is a fixed point of the operator $\mathcal{R}^{\pi,\mu}$. Because  $\mathcal{R}^{\pi,\mu}$ is also contractive by Proposition~\ref{prop:retracecontractive}, the fixed point is unique.
\end{proof}

\theoremquantilecontraction*
\begin{proof}
    The quantile projection $\Pi_\mathcal{Q}$ is a non-expansion under $\bar{W}_\infty$ \citep{dabney2018distributional}. Since $\mathcal{R}^{\pi,\mu}$ is $\beta$-contractive under $\bar{W}_p$ for all $p\geq 1$, the composed operator $\Pi_\mathcal{Q}\mathcal{R}^{\pi,\mu}$ is $\beta$-contractive under $\bar{W}_\infty$. Now, because (1) $\Pi_\mathcal{Q}\mathcal{R}^{\pi,\mu}\in\probspace_\infty(\mathbb{R})^{\mathcal{X}\times\mathcal{A}}$; (2) the space $\Pi_\mathcal{Q}\mathcal{R}^{\pi,\mu}\in\probspace_\infty(\mathbb{R})^{\mathcal{X}\times\mathcal{A}}$ is closed \citep{bdr2022}; (3) the operator is contractive, the iterate $\eta_k=\left(\Pi_\mathcal{Q}\mathcal{R}^{\pi,\mu}\right)^k\eta_0$ converges to a limiting distribution $\eta_\mathcal{R}^\pi\in\probspace_\infty(\mathbb{R})^{\mathcal{X}\times\mathcal{A}}$. Finally, by Proposition 5.28 in \citep{bdr2022}, we have $\bar{W}_\infty(\eta_\mathcal{R}^\pi,\eta^\pi)\leq (1-\beta)^{-1}\bar{W}_\infty(\Pi_\mathcal{Q}\eta^\pi,\eta^\pi)$.
\end{proof}

\lemmaunbiasedqr*
\begin{proof}
The QR loss $L_\theta^\tau(\eta)$ is defined for any distribution $\eta$ and scalar parameter $\theta$. Let $\nu=\sum_{i=1}^m w_i\eta_i$ be the linear combination of distributions $(\eta_i)_{i=1}^m$ where $w_i$s are potentially negative coefficients. In this case, $\nu$ is a signed measure. We define the generalized QR loss for $\nu$ as the linear combination of QR losses against $\eta_i$ weighted by $w_i$,
\begin{align*}
    L_\theta^\tau(\nu) \coloneqq \sum_{i=1}^m w_i L_\theta^\tau(\eta_i).
\end{align*}
Next, we note that the QR loss is linear in the input distribution (or signed measure). This means given any (potentially infinite) set of $N$ distributions or signed measures $\nu_i$ with coefficients $a_i$,
\begin{align*}
    L_\theta^\tau\left(\sum_{i=1}^N a_i \nu_i \right) = \sum_{i=1}^N a_i L_\theta^\tau( \nu_i).
\end{align*}
When $(a_i)_{i=1}^N$ denotes a distribution, the above is equivalently expressed as an exchange between expectation and the QR loss $L_\theta^\tau(\mathbb{E}[\nu_i])=\mathbb{E}[L_\theta^\tau(\nu_i)]$.
For notational convenience, we let $\theta=z_i(x,a)$ and $\tau=\tau_i$. Because the trajectory terminates within $H$ steps almost surely, since $c_{1:t}\leq \rho_{1:t}\leq \rho^H$ where $\rho\coloneqq\max_{x\in\mathcal{X},\mathcal{A}}\frac{\pi(a|x)}{\mu(a|x)}$, the estimate $\hat{L}_\theta^\tau\left(\mathcal{R}^{\pi,\mu}\eta(x,a)\right)$ is finite almost surely. Combining all results from above we obtain the following
\begin{align*}
    \mathbb{E}_\mu\left[\mathcal{R}^{\pi,\mu}\eta(x,a)\right] &=  \mathbb{E}_\mu\left[L_\theta^\tau \left(\eta(x,a)\right) + \sum_{t=0}^\infty c_{1:t} \left( L_\theta^\tau\left(\left(\textrm{b}_{t+1}\right)_{\#}\eta\left(X_{t+1},A_{t+1}^\pi\right)\right) - L_\theta^\tau\left(\left(\textrm{b}_t\right)_{\#}\eta(X_t,A_t)\right)\right)\right] \\
    &=_{(a)} \mathbb{E}_\mu\left[L_\theta^\tau\left(\hat{\mathcal{R}}^{\pi,\mu}\eta(x,a)\right)\right] =_{(b)} L_\theta^\tau\left(\mathbb{E}_\mu\left[\hat{\mathcal{R}}^{\pi,\mu}\eta(x,a)\right]\right) = L_\theta^\tau\left(\mathcal{R}^{\pi,\mu}\eta(x,a)\right).
\end{align*}
In the above, (a) follows from the definition of the generalized QR loss against signed measure the definition of $\mathcal{R}^{\pi,\mu}\eta(x,a)$; (c) follows from the linearity of the QR loss. 

Next, to show that the gradient estimate is unbiased too, the high level idea is to apply dominated convergence theorem (DCT) to justify the exhchange of gradient and expectation \citep{rudin1976principles}. Since the expected QR loss gradient $\nabla_\theta L_\theta^\tau\left(\mathcal{R}^{\pi,\mu}\eta(x,a)\right)$ exists, we deduce that the estimate $\nabla_\theta \hat{L}_\theta^\tau\left(\mathcal{R}^{\pi,\mu}\eta(x,a)\right)$ exists almost surely. Consider the absolute value of the gradient estimate $
    \left|\nabla_\theta \hat{L}_\theta^\tau\left(\mathcal{R}^{\pi,\mu}\eta(x,a)\right)\right|
$, which serves as an upper bound to the gradient estimate. In order to apply DCT, we need to show the expectation of the absolute gradient is finite. Note we have
\begin{align*}
    &\mathbb{E}_\mu\left[
    \left|\nabla_\theta \hat{L}_\theta^\tau\left(\mathcal{R}^{\pi,\mu}\eta(x,a)\right)\right|
\right] \\ &= \mathbb{E}_\mu\left[\left|\nabla_\theta L_\theta^\tau \left(\eta(x,a)\right) + \sum_{t=0}^H c_{1:t} \left( \nabla_\theta L_\theta^\tau\left(\left(\textrm{b}_{t+1}\right)_{\#}\eta\left(X_{t+1},A_{t+1}^\pi\right)\right) - \nabla_\theta L_\theta^\tau\left(\left(\textrm{b}_t\right)_{\#}\eta(X_t,A_t)\right)\right)\right|\right] \\
&\leq_{(a)} \mathbb{E}_\mu\left[\left|\nabla_\theta L_\theta^\tau \left(\eta(x,a)\right)\right| + \sum_{t=0}^H c_{1:t} \left| \nabla_\theta L_\theta^\tau\left(\left(\textrm{b}_{t+1}\right)_{\#}\eta\left(X_{t+1},A_{t+1}^\pi\right)\right) - \nabla_\theta L_\theta^\tau\left(\left(\textrm{b}_t\right)_{\#}\eta(X_t,A_t)\right)\right|\right] \\
&\leq_{(b)} \mathbb{E}_\mu\left[1 + \sum_{t=0}^H \rho^t \cdot 2\right] <\infty,
\end{align*}
where (a) follows from the application of triangle inequality; (b) follows from the fact that the QR loss gradient against a fixed distribution is bounded $\nabla_\theta L_\theta^\tau\left(\nu\right)\in[-1,1],\forall \nu\in\probspace_\infty(\mathbb{R})$ \citep{dabney2018distributional}.

With the application of DCT, we can exchange the gradient and expectation operator, which yields $ \mathbb{E}_\mu\left[\nabla_\theta \hat{L}_\theta^\tau\left(\mathcal{R}^{\pi,\mu}\eta(x,a)\right)\right] = \nabla_\theta \mathbb{E}_\mu\left[ \hat{L}_\theta^\tau\left(\mathcal{R}^{\pi,\mu}\eta(x,a)\right)\right] = \nabla_\theta L_\theta^\tau\left(\mathcal{R}^{\pi,\mu}\eta(x,a)\right)$.
\end{proof}

\end{document}